\documentclass[journal]{IEEEtran}
\usepackage{amsmath,amssymb,amsthm,url}
\usepackage{graphicx,subfigure}
\usepackage{verbatim}
\usepackage{color}
\usepackage{xcolor,tikz,pgfplots}
\usepackage{siunitx}
\usepackage[hidelinks]{hyperref}
\usepackage{lipsum} 
\usepackage{cleveref}
\usepackage{booktabs}
\usepackage{algorithm}
\usepackage{algpseudocode}
\usepackage{adjustbox}
\usepackage{multirow}
\usepackage{cite}
\usepackage{colortbl} 
\pgfplotsset{compat=1.18}

\definecolor{findOptimalPartition}{HTML}{D7191C}
\definecolor{storeClusterComponent}{HTML}{ffb75f}
\definecolor{dbscan}{HTML}{ABDDA4}
\definecolor{constructCluster}{HTML}{2B83BA}

\newcommand{\SO}{\ensuremath{\mathsf{SO(3)}}}

\newcommand{\so}{\ensuremath{\mathfrak{so}(3)}}

\renewcommand{\Re}{\ensuremath{\mathbb{R}}}
\newcommand{\Sph}{\ensuremath{\mathsf{S}}}

\newcommand{\E}{\mathbb{E}}
\newcommand{\grav}{\mathsf{g}}
\newcommand{\mono}{\textrm{mono}}
\renewcommand{\mod}{\textrm{mod}}

\title{
	Equivariant Reinforcement Learning Frameworks for Quadrotor Low-Level Control
}
\author{Beomyeol Yu and Taeyoung Lee
	\thanks{Beomyeol Yu, Taeyoung Lee, Mechanical and Aerospace Engineering, George Washington University, Washington, DC 20052, {\tt \{yubeomyeol,tylee\}@gwu.edu}}%
	\thanks{\textsuperscript{\footnotesize\ensuremath{*}}This research has been supported in part by NSF CNS-1837382, AFOSR MURI FA9550-23-1-0400, and ONR N00014-23-1-2850.}
}

\newtheorem{prop}{Proposition}
\newtheorem{remark}{Remark}

\graphicspath{{./figs/}}

\makeatletter
\def\endthebibliography{%
	\def\@noitemerr{\@latex@warning{Empty `thebibliography' environment}}%
	\endlist
}
\makeatother

\begin{document}
	\allowdisplaybreaks
	\maketitle \thispagestyle{empty} \pagestyle{empty}
	
    \begin{abstract}
        Improving sampling efficiency and generalization capability is critical for the successful data-driven control of quadrotor unmanned aerial vehicles (UAVs) that are inherently unstable.
        While various reinforcement learning (RL) approaches have been applied to autonomous quadrotor flight, they often require extensive training data, posing multiple challenges and safety risks in practice. 
        To address these issues, we propose data-efficient, equivariant monolithic and modular RL frameworks for quadrotor low-level control.
        Specifically, by identifying the rotational and reflectional symmetries in quadrotor dynamics and encoding these symmetries into equivariant network models, we remove redundancies of learning in the state-action space.
        This approach enables the optimal control action learned in one configuration to automatically generalize into other configurations via symmetry, thereby enhancing data efficiency.
        Experimental results demonstrate that our equivariant approaches significantly outperform their non-equivariant counterparts in terms of learning efficiency and flight performance.
    \end{abstract}

	\section{Introduction}
	Achieving precise and robust control for unmanned aerial vehicles presents significant challenges due to their complexity and nonlinearity of the dynamics, and sensitivity to environmental disturbances.
	Reinforcement learning (RL), with its capability to learn data-driven control policies through experience, has emerged as a powerful paradigm for addressing these challenges. 
	Unlike traditional control methods~\cite{bouabdallah2004pid,xu2006sliding,lee2010geometric} that require precise modeling and extensive analysis, model-free RL schemes learn optimal control policies through an interactive learning process without requiring an exact mathematical model.

	Prior works in model-free RL for quadrotor control have predominantly focused on developing an end-to-end, monolithic RL policy designed to manage all aspects of the quadrotor dynamics through a single, generalized RL agent.
	The foundational work by~\cite{hwangbo2017control} introduced a deep RL framework for quadrotor control, integrating low-gain proportional-derivative (PD) controllers for attitude stabilization alongside learned policies.
	To further enhance tracking accuracy, stochastic and deterministic RL policies were explored in~\cite{lopes2018intelligent} and~\cite{shehab2021low}, respectively.

	To bridge the gap between numerical simulations in virtual environments and actual flights in the real world, the transferability of a model-free monolithic control policy was demonstrated by~\cite{molchanov2019sim} without relying on PD controllers.
	Further, autonomous landing on a moving platform has been investigated in~\cite{rodriguez2019deep}, and quadrotor flights through tilted narrow gaps were demonstrated in~\cite{xiao2021flying}, highlighting the potential of RL-based control in dynamic environments.
	Meanwhile, a multi-purpose low-level control policy is presented in \cite{pi2021general}, demonstrating the capability of controlling both quadrotor and hexacopter UAVs with the identical policy, thereby addressing the limitations of model-specific RL control policies. 
	In~\cite{pi2021robust}, disturbance compensation techniques were incorporated to enhance robustness against external disturbances such as wind gusts.
	Additionally, RL-based control policies are benchmarked in~\cite{kaufmann2022benchmark} for agile drone racing. 
	
	Despite these advancements, the common monolithic RL policies often exhibit limitations in scenarios requiring precise yaw control, such as high-speed drone racing.
	Most of the prior studies have overlooked yaw control during training and real-world deployment.
	While \cite{kaufmann2023champion} attempted to address this by incorporating yaw error into their reward function, their approach was assisted by a PID controller when transferring their monolithic RL policy to real-world settings.
	Motivated by these limitations, recent works in~\cite{yu2024multi} and \cite{yu2024modular} proposed \textit{modular reinforcement learning} strategies that geometrically decouple translational and yaw dynamics. 
	These modular policies, validated through zero-shot sim-to-real transfer, significantly enhance both learning efficiency and flight performance.
	
	Nonetheless, these symmetry-agnostic RL methods remain inherently data-intensive, relying on deep neural networks that often require extensive training samples to develop robust control policies, leading to slower learning rates.
	This data inefficiency is particularly problematic in aerial vehicles, where data collection poses higher risks and costs than ground-based systems.
	As a result, achieving reliable and efficient RL training for quadrotors, which involve complex, high-dimensional data, remains a significant challenge.
	These limitations have motivated research into sample-efficient RL techniques, including model-based RL~\cite{lambert2019low} and curriculum learning~\cite{eschmann2024learning}.
	
	An alternative, promising approach to enhance data efficiency involves \textit{equivariant learning}, rooted in the foundational principles of geometric learning~\cite{bronstein2017geometric}.
	Equivariant learning focuses on encoding domain-specific symmetries directly into neural network architectures, thereby improving sample efficiency.
	By embedding geometric relationships (e.g., translations, rotations, and permutations) between the input and output data, symmetry-based equivariant networks reduce the redundancy of state-action pairs during exploration.
	Early advancements, including Group equivariant Convolutional Neural Networks (G-CNNs)~\cite{cohen2016group} and Steerable CNNs~\cite{cohen2016steerable}, demonstrated the effectiveness of group convolution layers in capturing symmetries for computer vision tasks.
	Expanding these ideas, a general framework for implementing $\ensuremath{\mathsf{E(2)}}$-equivariant networks was developed in~\cite{weiler2019general}.
	These approaches have been extended to vision-based robotic manipulation tasks in~\cite{wang2022mathrm,wang2022robot}, where embedding equivariance into Q-functions and policies of RL significantly enhanced learning efficiency.
	In \cite{tangri2024equivariant}, $\ensuremath{\mathsf{SO(2)}}$-equivariance was utilized by combining Conservative Q-Learning (CQL) and Implicit Q-Learning (IQL) with symmetry-aware models.
	Additionally, in \cite{van2020mdp,van2021multi}, Markov Decision Process (MDP) homomorphic networks were introduced to exploit symmetries in the joint state-action space of MDPs, specifically reflectional and rotational equivariance, to reduce sample complexity.

	In applications to the quadrotor dynamics, equivariant RL models have demonstrated substantial improvements in sample efficiency and policy generalization.
	In~\cite{huang2023symmetry}, a symmetry-informed model-based RL approach was proposed for the attitude control of the quadrotor, leveraging the symmetry of angular velocity and torque with respect to body-fixed planes.
	Furthermore, in our previous work~\cite{yu2023equivariant}, we developed an $\Sph^1$-equivariant RL framework for quadrotor low-level control that directly maps the quadrotor's state to motor control signals.
	Specifically, we identified a rotational symmetry, where the optimal control represented in the body-fixed frame remains invariant under rotations about the gravity direction.
	This structural property was then embedded into an actor-critic architecture, reducing the dimensionality of the training data by one.
	
    Building on these recent advancements in equivariant RL, this study introduces novel monolithic and modular equivariant RL frameworks for the data-efficient training of low-level quadrotor control policies.
    By leveraging the inherent symmetries of quadrotor dynamics, the proposed methods eliminate redundancies in learning and accelerate convergence.
    Specifically, we present two equivariant RL frameworks: the monolithic equivariant RL that incorporates rotational symmetry, and the modular equivariant RL that exploits both rotational and reflectional symmetries for enhanced yaw control.
    Notably, the modular framework not only improves training efficiency but also addresses the limitations of monolithic architectures in handling coupled control objectives.
    Experimental results demonstrate that the proposed frameworks outperform non-equivariant baselines in terms of learning convergence and control robustness within a given training period.
    This study advances the state-of-the-art in RL-based quadrotor control and establishes a foundation for extending geometric RL methods to other robotic systems with underlying symmetries.

	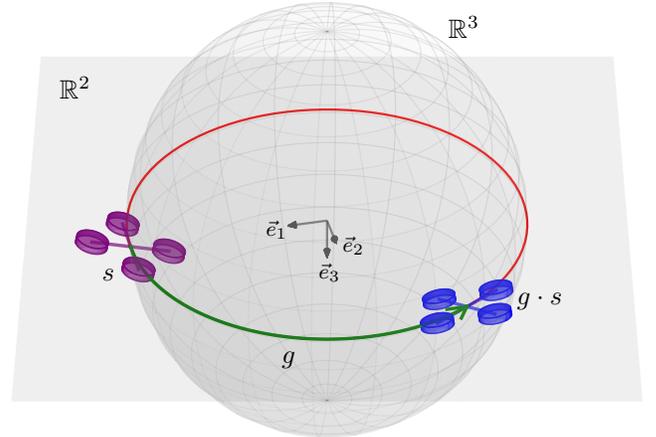
\begin{figure}[t]
		\centering
		\vspace*{0.8cm}
		\begin{picture}(225,150)
		\put(0,0){\includegraphics[width=0.47\textwidth]{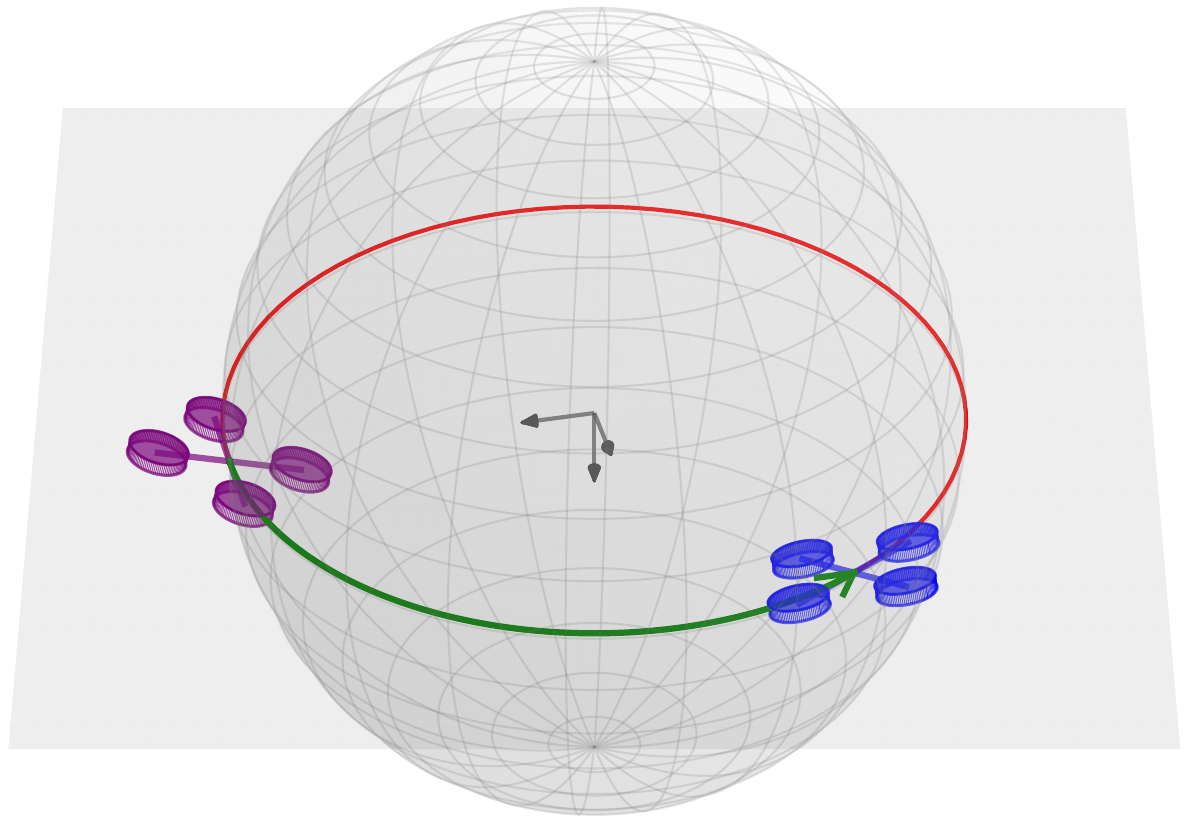}} 
		\put(98,78){{\footnotesize $\vec e_1$}}
		\put(127,72){{\footnotesize $\vec e_2$}}
		\put(118,62){{\footnotesize $\vec e_3$}}		
		\put(36,61){{\normalsize $s$}}
		\put(104,30){{\normalsize $g$}}
		\put(193,52){{\normalsize $g \cdot s$}}
		\put(20,130){{\normalsize $\Re^2$}}
		\put(167,153){{\normalsize $\Re^3$}}
		\end{picture}
		\vspace*{0.1cm}
		\caption{Illustration of the group action corresponding to the rotation about the vertical axis $\vec e_3$.
			A group element $g \in \SO_{\vec e_3}$ (green), which corresponds to $\ensuremath{\mathsf{SO(2)}}$ embedded in $\SO$ as a subgroup by fixing the axis of rotation as $\vec e_3$, acts on a quadrotor state $s$ (purple) by rotating it to a new state $g \cdot s$ (blue). 
			The orbit of $s$ under this action, denoted by $G \cdot s$, is the set of all points reachable by rotations about $\vec e_3$, and its projection on the position space corresponds to a circle (red).
            In the proposed equivariant RL, the control policy learned at a single point on the orbit is automatically generalized to any other points on the orbit. 
		} 
		\label{fig:rot_sym_concept} 
		\vspace*{-0.3cm}
	\end{figure}	
	
	In short, the main contributions of this work are as follows:
	\begin{enumerate}
		\item We develop data-efficient monolithic and modular reinforcement learning frameworks for quadrotor low-level control by exploiting group-equivariant neural networks that respect the rotational and reflectional symmetries of the quadrotor dynamics, where the generalization capability is naturally encoded. 
		This approach reduces the dimensionality of sample trajectories, thereby enhancing convergence rates.
		\item We compare monolithic and modular architectures, highlighting that the modular design, which decouples translational and yaw dynamics, significantly enhances both tracking accuracy and learning efficiency.
		\item We demonstrate the superiority of the proposed equivariant frameworks over traditional non-equivariant counterparts through numerical simulations and real-world flight experiments.
	\end{enumerate}
	
	The remainder of the paper is structured as follows.
    \Cref{sec:BG} introduces the background of RL and equivariance learning, and \Cref{sec:SymDyn} explores symmetries in the quadrotor dynamics.
    Next, in \Cref{sec:EquivRL}, the proposed equivariance monolithic and modular RL frameworks are introduced. 
	Finally, experimental results and conclusions are presented in \Cref{sec:Exp,sec:Conc}, respectively.

	\section{Backgrounds}\label{sec:BG}	
	\subsection{Reinforcement Learning Problems}
	In reinforcement learning, a decision-making agent aims to learn a policy through interactions with an environment, guided by reward signals.
	Markov Decision Process (MDP) serves as the mathematical foundation for modeling RL problems.  
    It is defined by the tuple $(\mathcal{S}, \mathcal{A}, \mathcal{R}, \mathcal{T}, \gamma)$, where $\mathcal{S}$ and $\mathcal{A}$ denote state and action spaces, respectively.
	Next, the reward function $\mathcal{R}: \mathcal{S} \times \mathcal{A} \rightarrow \mathbb{R}$ maps state-action pairs to scalar reward values, defined by $r(t) = \mathcal{R}(s(t), a(t))$ at time $t$.
	The transition probability function, $\mathcal{T}: \mathcal{S} \times \mathcal{A} \rightarrow \mathcal{P}(\mathcal{S})$, defines the distribution $p(s(t+1) | s(t), a(t))$ of the new state $s(t+1)$ at $t+1$, given a current $s(t)$ under an action $a(t)$.
	
	The goal of RL agents is to find an optimal policy $\pi^* (a(t) | s(t))$ that maximizes the expected discounted return,
	\begin{align}
		J(\pi) = \E_{\tau \sim \pi} [\sum_{k=t}^\infty \gamma^{k-t} r(s(k), a(k))], \label{eqn:J}	
	\end{align}
	where $\gamma \in (0, 1)$ is a discount factor for the temporal decay of future rewards, and $\tau = \{s_0, a_0, s_1, a_1, \ldots\}$ denotes a trajectory of state-action pairs over the interval $[t,\infty)$ starting at $s(t) = s_0\in\mathcal{S}$. 
	Then, the optimal policy $\pi^*$ can be expressed as 
	\begin{align}
		{\pi}^* = \underset{\pi}{\arg\max} ~ J(\pi).
	\end{align}
	
	For a given policy $\pi$, the optimality is often formulated through a state-value function $V_\pi : \mathcal{S} \rightarrow \Re$, which represents the expected return when starting in $s_0$ and following policy $\pi$, defined by
	\begin{align}
		V_\pi(s(t)) = \E_{\tau \sim \pi} [\sum_{k=t}^\infty \gamma^{k-t} r(s(k), a(k)) | s_0]. \label{eqn:V_discrete}	
	\end{align}
	Alternatively, it is useful to introduce an action-value function, also known as a Q-value, $Q_\pi: \mathcal{S} \times \mathcal{A} \rightarrow \Re$.
	This represents the expected return when starting at $s_0$, taking an action $a(t) = a_0\in\mathcal{A}$, and subsequently following policy $\pi$, defined as
	\begin{align}
		Q_\pi(s(t),a(t)) = \E_{\tau \sim \pi}  [\sum_{k=t}^\infty \gamma^{k-t} r(s(k), a(k)) | s_0, a_0]. \label{eqn:Q_discrete}	
	\end{align}

	Further, in continuous-time, deterministic MDPs, the above discrete-time value function \eqref{eqn:V_discrete}  is reformulated as
	\begin{gather}
	V_\pi(s(t)) = \int_t^\infty \gamma^{\tau-t} r(s(\tau), a(\tau)) \, d\tau, \label{eqn:V_continuous} 
	\end{gather}
	for $\tau \in [t, \infty)$. 
    Typically, the value functions $V_\pi(s)$ and $Q_\pi(s, a)$ are represented by neural networks and recursively updated by the Bellman equation.
	
	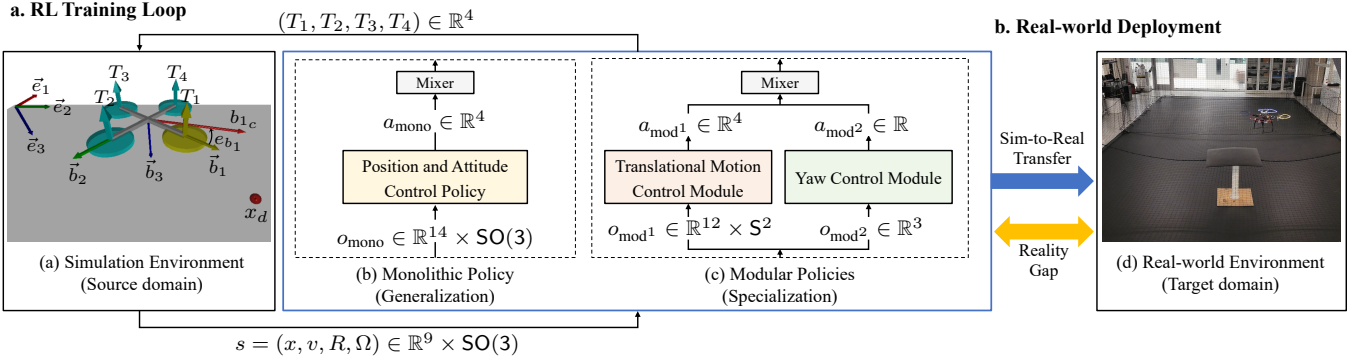
\begin{figure*}[t]
		\centering
		\hspace*{-0.55cm}
		\begin{picture}(488,125)
			\put(0,0){\includegraphics[width=0.99\textwidth]{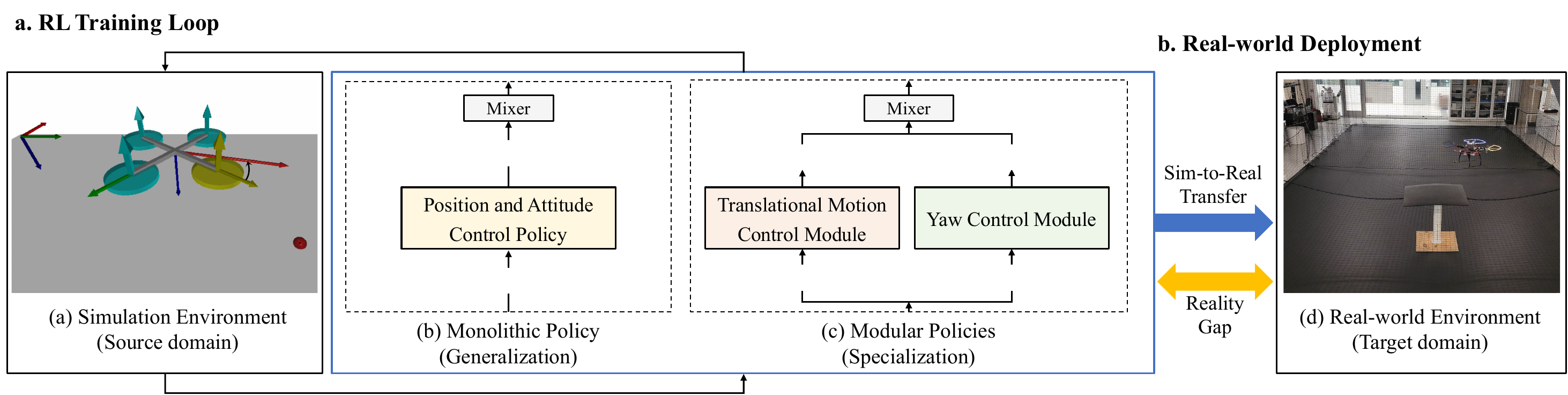}}
			
			\put(13,90){{\scriptsize $\vec e_1$}}
			\put(21,82){{\scriptsize $\vec e_2$}}
			\put(11,67.5){{\scriptsize $\vec e_3$}}		
			\put(82,70){{\scriptsize $e_{b_1}$}}
			\put(88,78){{\scriptsize $b_{1_c}$}}
			\put(80,59){{\scriptsize $\vec b_1$}}
			\put(27,56){{\scriptsize $\vec b_2$}}
			\put(56,57){{\scriptsize $\vec b_3$}}
			\put(69,86){{\scriptsize $T_1$}}
			\put(36.5,85.5){{\scriptsize $T_2$}}
			\put(42,95){{\scriptsize $T_3$}}
			\put(64,95){{\scriptsize $T_4$}}
			\put(93.5,42){{\footnotesize $x_d$}}
			
			\put(106,114){{\footnotesize $(T_1, T_2, T_3, T_4) \in \Re^4$}}
			\put(90,-8){{\footnotesize $s=(x, v, R, \Omega) \in \Re^{9} \times \SO$}}
			
			\put(146,76){{\footnotesize $a_\mono \in \Re^4$}}
			\put(128.5,32){{\footnotesize $o_\mono \in \Re^{14} \times \SO $}}
			
			\put(242.5,75){{\footnotesize $a_{\mod^1} \in \Re^4$}}
			\put(231,35.5){{\footnotesize $o_{\mod^1} \in \Re^{12}\times \Sph^2$}}
			\put(309,75){{\footnotesize $a_{\mod^2} \in \Re$}}
			\put(311,35.5){{\footnotesize $o_{\mod^2} \in \Re^3$}}
		\end{picture}
		\vspace*{0.35cm}
		\caption{
			A schematic overview of the system.
			\textbf{a.}, During training, we train RL policies for quadrotor low-level control tasks in simulation. 
			(a) A custom simulator serves as a training environment, providing full access to the quadrotor's dynamics and state.
			(b) A monolithic end-to-end policy directly outputs total thrust $f$ and moments $M$.
			(c) Two specialized modules independently control translational and yaw motions, each selecting the optimal action based on its local observations.
			\textbf{b.}, When transferring trained policies from simulation to the physical world, the sim-to-real gap arises from mismatches between simulation and reality.
			To bridge this gap, domain randomization is applied during the training phase.
			(d) An indoor flight test facility at the Flight Dynamics and Control Lab, GWU for real-world deployment.
			A supplementary video of the RL training and real-world experiments is available at \url{https://youtu.be/TGBQTuKpbAw}.
		}
		\label{fig:framework}
		\vspace*{-0.2cm}
	\end{figure*}

    \subsection{Equivariant Neural Networks} 	
    Equivariant learning is a machine learning paradigm where the model’s predictions change in a structured, predictable way in response to transformations of the input. 
    In simpler terms, if the input to the model is transformed by a group action, e.g., rotation, translation, and reflection, the model’s output should change correspondingly in a consistent manner.
	
	More precisely, a \textit{group} $G$ is a set equipped with a binary operation satisfying closure, associativity, identity, and inverse properties. 
	A (left) action of a group $G$ on $X$ is a mapping $\Phi: G \times X \rightarrow X$ satisfying two conditions: 
    for any $x\in X$ and $g,h\in G$, $\Phi(e,x)=x$ and $\Phi(g,\Phi(h,x))=\Phi(gh,x)$, where $e\in G$ is the identity element. 
    This notation for the group action is often shortened into $\Phi(g,x)=g \cdot x = gx$. 
    Next, suppose $G$ acts on both $X$ and $Y$. 
    A map $f:X\rightarrow Y$ is \textit{equivariant} if 
    \begin{equation}
        f(g\cdot x) = g \cdot f(x), \label{eqn:equiv}
    \end{equation}
    for any $g\in G$ and $x\in X$. 
    As discussed above, this implies that the action on the input of $f$ is equivalent to the action on the output of $f$. 
    In other words, a certain transformation in the input results in the corresponding transformation in the output, implying a structured relation between the input space and the output space, which can be utilized for the generalization capability of any data-driven technique. 

	When the space $X$ on which the group acts is a vector space $V$ and the action is linear, the group action can be described by a linear transformation, i.e., matrix multiplication. 
	Specifically, the \textit{representation} of $G$ on a vector space $V$ is a map $\rho$ from $G$ to the general linear group $GL(V)$ on $V$ such that $\rho(g_1,g_2) = \rho(g_1)\rho(g_2)$ for any $g_1, g_2\in G$. 
	Since $\rho(g)\in GL(V)$, the representation is an invertible linear transformation from $V$ to itself, and as such, the group action can be written as the matrix multiplication $\Phi(g, v) = \rho(g) v$ for $v\in V$.  
		
	Neural network models satisfying the equivariant property of \eqref{eqn:equiv} are referred to as equivariant neural networks. 
	It is well known that the classical CNNs satisfy the translational equivariance~\cite{lecun1989backpropagation}, implying that they can detect the same feature in the image when the location of the feature is shifted. 
	Further, equivariant neural networks have been extended to other symmetries, such as rotations and reflections. 
	For instance, as demonstrated by~\cite{wang2022mathrm}, incorporating rotational symmetry into RL models substantially improved sample efficiency in robotic manipulation tasks.
	Recently, equivariant multilayer perceptrons (EMLPs) for multiple matrix groups were presented in~\cite{finzi2021practical}, which is utilized later in this paper when implementing equivariant RL models.

	\section{Symmetries in Quadrotor Dynamics}\label{sec:SymDyn}
	In this section, we explore the symmetry properties of quadrotor dynamics and discuss how they yield equivariance properties that can be utilized in RL. 
    We first show the rotational symmetry of the complete, coupled dynamics of the quadrotor.
    Second, we introduce another formulation of the quadrotor dynamics that are decomposed into the translational part and the yawing part, and show the rotational symmetry and the reflection symmetry, respectively for each part. 
    These two types of quadrotor dynamic models are referred to as the monolithic model and the modular model, respectively, following the name of RL frameworks that will be developed in \Cref{sec:EquivRL}.
	
	\subsection{Monolithic Model} \label{sec:dyn_mono}	

	Consider the inertial frame $\{\vec e_1,\vec e_2,\vec e_3\}$, where $\vec e_3$ is aligned with the gravity pointing downward, and the body-fixed frame $\{\vec b_{1},\vec b_{2},\vec b_{3}\}$ located at the mass center of the quadrotor, where $\vec b_{3}$ points in the opposite direction to thrust.
    The state transition dynamics  $\mathcal{T}$ is constructed by discretizing the following equations of motion for the quadrotor~\cite{lee2010geometric}:
	\begin{gather}
		\dot x  = v,\label{eqn:x_dot}\\
        m \dot v = m\grav e_3 - f R e_3,\label{eqn:v_dot}\\
		\dot R = R\hat\Omega,\label{eqn:R_dot}\\
		J\dot \Omega + \Omega\times J\Omega = M,\label{eqn:W_dot}
	\end{gather}
	where the \textit{hat map}, denoted by $\hat\cdot:\Re^3\rightarrow\so=\{S\in\Re^{3\times 3}\,|\, S^T = -S\}$,  transforms a vector in $\Re^3$ to a $3\times 3$ skew-symmetric matrix, such that $\hat x y = x\times y$ for any $x,y\in\Re^3$.
    We denote the mass and inertia matrix of the quadrotor by $m\in\Re$ and $J\in\Re^{3\times 3}$, respectively, and the gravitational acceleration by $\grav\in\Re$. 
    Also, $e_3=[0,0,1]^T\in\Re^3$. 

	Here, the position and velocity of the quadrotor in the inertial frame are denoted by $x\in\Re^3$ and $v\in\Re^3$, respectively.
	Also, the attitude is defined by the rotation matrix $R\in\SO=\{R\in\Re^{3\times 3}\,|\, R^T R=I_{3\times 3},\; \mathrm{det}[R]=1\}$, which is the transformation of the representation of a vector from the body-fixed frame to the inertial frame, and the angular velocity resolved in the body-fixed frame is denoted by $\Omega = [\Omega_1, \Omega_2, \Omega_3]^T \in\Re^3$.
	Thus, the state variable for the monolithic model is defined as $s_\mono = (x, v, R, \Omega) \in \mathcal{S}_\mono = \Re^{9}\times \SO$.

	Next, the action variable is defined as a four-dimensional vector composed of the total thrust $f \in\Re$ and the control moment $M = [M_1, M_2, M_3]^T \in\Re^3$ resolved in the body-fixed frame, and it is denoted by $a_\mono = (f, M)\in \mathcal{A}_\mono = \Re^4$.
	Here, the thrust of each motor $T_i$ is determined by the following mixer matrix,
	\begin{gather} \label{eqn:mixer}
		\begin{bmatrix} 
			T_1 \\ T_2 \\ T_3 \\ T_4
		\end{bmatrix}
		= \frac{1}{4}
		\begin{bmatrix}
			1 & 0      & 2/d   & -1/c_{\tau f} \\
			1 & -2/d & 0      & 1/c_{\tau f} \\
			1 & 0      & -2/d & -1/c_{\tau f} \\
			1 & 2/d   & 0      & 1/c_{\tau f} 
		\end{bmatrix}
		\begin{bmatrix}
			f \\ M_1 \\ M_2 \\ M_3 
		\end{bmatrix}.
	\end{gather}
	where $d\in\Re$ represents the distance between $\vec b_{3}$ and the center of any rotor, and $c_{\tau f} \in\Re$ is a constant that relates the thrust and the resulting reactive torque.

    \paragraph*{Rotational Equivariance}
    As illustrated in \Cref{fig:rot_sym_concept}, consider the group $G = \SO_{\vec e_3}$, which corresponds to the group of planar rotations $\ensuremath{\mathsf{SO(2)}}$ embedded in $\SO$ as a subgroup by fixing the axis of rotation to be the third inertial axis $\vec e_3$.
    The group action of $G$ on $\Re^3$ or $\SO$ can be expressed by matrix multiplication, with the following representation. 
    Let $\rho_\theta:G\rightarrow\Re^{3\times 3}$ be a group representation of $G$ parameterized by the rotation angle $\theta\in(-\pi, \pi]$.
    Its action on $v\in\Re^3$ can be written as
    \begin{align}
        \rho_\theta(g)v =  
        \exp(\theta\hat e_3) v = 
        \begin{bmatrix}
            \cos \theta & -\sin \theta & 0 \\ 
            \sin \theta & \cos \theta & 0 \\
            0 & 0 & 1
        \end{bmatrix}
        v.\label{eqn:rho_theta_g}
    \end{align}

	Next, we demonstrate that the quadrotor’s full dynamics \eqref{eqn:x_dot}--\eqref{eqn:W_dot} is symmetric with respect to $G$. 
    The group action of $G$ on $\mathcal{S}_\mono$ and $\mathcal{A}_\mono$ is defined by
	\begin{align}
		g s_\mono & = (\rho_\theta(g) x, \rho_\theta(g) v, \rho_\theta(g) R, \Omega), \label{eqn:gs_mono} \\
		g a_\mono & = (f, M). \label{eqn:ga_mono}
	\end{align}
	This corresponds to rotating the complete system about the vertical axis $\vec e_3$ by the angle $\theta$. 
	In other words, $g$ operates on $(s_\mono, a_\mono)$ by rotating $x$, $v$, and $R$, but leaves $\Omega$, $f$, and $M$ unchanged.
	Here, $\Omega$ and $M$ remain unaffected by the rotation as they are resolved with respect to the body-fixed frame. 
	However, when perceived from the inertial frame, the angular velocity is rotated from $\omega \triangleq R \Omega\in\Re^3$ into $\rho_\theta(g)R\Omega = \rho_\theta(g) \omega$ by the group action.
    Similarly, $M$ is actually rotated in the inertial frame. 
	Also, $f$ is not changed by the rotation, as it is the total thrust magnitude.
    Note that we adopt an abuse of notation, where the group action on the state $s_\mono$ and the group action on the control action $a_\mono$ are denoted by the same symbol $g$, i.e., it is understood that \eqref{eqn:gs_mono} or \eqref{eqn:ga_mono} is chosen depending on the proper context. 
	
    Next, we show that the monolithic dynamic model is equivariant. 
    \begin{prop}\label{prop:sym_mono}
        Let the equations of motion \eqref{eqn:x_dot}--\eqref{eqn:W_dot} be consolidated into
        \begin{align}
            \dot s_\mono = F(s_\mono, a_\mono),\label{eqn:s_dot}
        \end{align}
        where $F:\mathcal{S}_\mono\times \mathcal{A}_\mono\rightarrow T\mathcal{S}_\mono$, and $T\mathcal{S}_\mono$ denotes the tangent bundle of the state space.  
        Then $F$ is equivariant with respect to the action defined in \eqref{eqn:gs_mono} and \eqref{eqn:ga_mono}, i.e., 
        \begin{align}
            F \circ g = g \circ F,
        \end{align}
        where the action on $T\mathcal{S}_\mono$ is defined as \eqref{eqn:gs_mono}.
    \end{prop}
	
	\begin{proof}
		From \eqref{eqn:x_dot}--\eqref{eqn:W_dot}, we have
		\begin{align*}
		F(s_\mono, a_\mono) = 
        \begin{bmatrix} v\\
            \grav e_3 - \frac{f}{m} R e_3 \\
				R \hat{\Omega}\\
                J^{-1}(M - \Omega \times J \Omega)
            \end{bmatrix}.
		\end{align*}
        Then, $F\circ g$ is 
		\begin{align*}
            F (gs_\mono, g a_\mono) = 
            \begin{bmatrix}
				\rho_\theta(g) v\\
                \grav e_3 - \frac{f}{m} \rho_\theta(g) R e_3\\
				\rho_\theta(g) R \hat{\Omega}\\ 
                J^{-1}(M - \Omega \times J \Omega)
            \end{bmatrix}. 
        \end{align*}
        This is identical to $g\circ F$ as $\rho_\theta(g) e_3 = e_3$ for any $\theta$. 
	\end{proof}
	
	This implies that if the state-action trajectory $(s(t), a(t))$ for $t \in [0,T]$ is a solution to the quadrotor dynamics given by \eqref{eqn:x_dot}--\eqref{eqn:W_dot}, then its rotated trajectory $(gs(t), ga(t))$ is another trajectory of the same dynamics for any $\theta$.

	This rotational symmetry further allows us to define an equivalence relation on the set of state-action trajectories over the interval $[0,T]$. 
	More specifically, we define a relation:
	\begin{align*}
		(s(t), a(t)) \sim (\tilde{s}(t), \tilde{a}(t)) 
	\end{align*}
	if there exists $g\in G$ such that $(\tilde{s}(t), \tilde{a}(t)) = (gs(t), ga(t))$ for all $t \in [0,T]$, and this relation satisfies the properties of the equivalence relation, e.g.,
	\begin{itemize}
		\item Reflexivity: $(s(t), a(t)) \sim (s(t), a(t))$ by taking $\theta = 0$.
		\item Symmetry: If $(s(t), a(t)) \sim (\tilde{s}(t), \tilde{a}(t))$ with $\theta$, then $(\tilde{s}(t),$ $\tilde{a}(t)) \sim (s(t), a(t))$ by using the inverse rotation $-\theta$.
        \item Transitivity: If $(s_1(t), a_1(t)) \sim (s_2(t), a_2(t))$ with $\theta_{12}$ and $(s_2(t),$ $a_2(t)) \sim (s_3(t), a_3(t))$ with $\theta_{23}$, then $(s_1(t), a_1(t)) \sim (s_3(t), a_3(t))$ by composing the two rotations to form $\theta_{13} = \theta_{12} + \theta_{23}$.
	\end{itemize}
	Given this equivalence relation, we define the equivalence class of the trajectory for any pair $(s, a) \in \mathcal{S}_\mono \times \mathcal{A}_\mono$ as
	\begin{align*}
		[s, a] = \{ (\tilde{s}, \tilde{a}) \in \mathcal{S}_\mono \times \mathcal{A}_\mono \mid (s, a) \sim (\tilde{s}, \tilde{a}) \}.
	\end{align*}
	Thus, the behavior of the quadrotor can be fully captured on the quotient space $\mathcal{S}_\mono \times \mathcal{A}_\mono /\sim$.
    In other words, the dimension of the domain in which the action-value function and the policy should be trained is reduced by the dimension of $G$, which is one in the presented quadrotor dynamics.
    Or, this can also be interpreted that the learning at any specific state-action pair $(s, a)$ is automatically generalized into any other points on its equivalence class $[s, a]$. 
    These improve the data-efficiency and the generalization capability of RL as discussed in the preceding sections. 
	
    \subsection{Modular Model}\label{sec:dyn_mod} 

	While reinforcement learning approaches based on the above monolithic model are widely used in quadrotor controls, they often suffer from performance degradation in flights involving agile yaw maneuvers.
	This is primarily because a single policy should manage the complete dynamics, including the yawing motion that has unique characteristics. 
	To address this, a modular reinforcement learning approach has been proposed in~\cite{yu2024multi}, where the translational dynamics are decoupled from the yaw dynamics.
	The fundamental idea behind decomposing the quadrotor dynamics involves splitting the three-dimensional orthogonal group $\SO$ of the attitude into the two-sphere $\Sph^2$, and the one-sphere $\Sph^1$, where the former represents the direction of the thrust $b_3$ corresponding to pitching and rolling, and the latter corresponds to the rotation about $b_3$ representing yawing. 
    Here, we investigate the equivariance properties of each component of the modular model to be utilized in the subsequent development of the equivariant, modular RL. 
	
    \subsubsection{Translational Dynamics Module}
	This module covers the translational motion of the quadrotor, including the roll and pitch dynamics that govern the direction of the total thrust. 
	In the equations of motion \eqref{eqn:x_dot} and \eqref{eqn:v_dot} presented in \Cref{sec:dyn_mono}, the translational dynamics are coupled with the attitude dynamics solely through the single term $b_3 \triangleq Re_3\in\Sph^2$. 
    The rotational dynamics, given by \eqref{eqn:R_dot} and \eqref{eqn:W_dot}, are not affected by the translational dynamics. 
	Assuming the quadrotor is inertially symmetric about the third body-fixed axis, i.e., $J_1 = J_2$, the dynamics of $b_3$ can be separated from the full attitude dynamics \eqref{eqn:R_dot} and \eqref{eqn:W_dot}.
	This yields the following translational dynamics, decoupled from the yaw~\cite{gamagedara2019geometric,gamagedara2022geometric}: 
	\begin{gather}
		\dot x  = v,\label{eqn:decoup_x_dot}\\
		m \dot v = mge_3 - f b_3,\label{eqn:decoup_v_dot}\\
		\dot b_3 = \omega_{12} \times b_3,\label{eqn:decoup_b3_dot}\\
		J_1 \dot \omega_{12} = \tau,\label{eqn:decoup_w12_dot}
	\end{gather}
	where $\omega_{12} = \Omega_1 b_1 + \Omega_2 b_2 \in \Re^3$ denotes the angular velocity of $b_3$ resolved in the inertial frame, satisfying $\omega_{12}  \perp b_3$ and $\dot\omega_{12}  \perp b_3$ always.
	Also, $\tau = \tau_1 b_1 + \tau_2 b_2 \in \Re^3$ represents a fictitious control moment defined by $\tau_1,\tau_2\in\Re$, which are related to the first two components of the actual control moment $(M_1, M_2)$ through 
	\begin{align*}
		M_1 = \tau_1 + J_3\Omega_2\Omega_3, \quad M_2 = \tau_2 - J_3\Omega_3\Omega_1.
	\end{align*}
	While $\tau$ consists of three elements, it is constrained to two degrees of freedom due to the constraint, $\tau \perp b_3$. 

	In summary, the roll and pitch dynamics for $b_3$ evolving on $\Sph^2$ are controlled by the control moment $\tau$ defined in terms of $(M_1, M_2)$, and they determine the direction of the total thrust, namely $-b_3$. 
	Given the magnitude of the total thrust $f\in\Re$, the complete thrust vector controlling the translational dynamics is given by $-fb_3\in\Re^3$ resolved in the inertial frame. 
	Thus, the state and action variables for this first, translational dynamics module are defined as $s_{\mod^1}  = (x, v, b_3, \omega_{12}) \in \mathcal{S}_{\mod^1}  = \Re^{9}\times \Sph^2$ and $a_{\mod^1}  = (f, \tau) \in \mathcal{A}_{\mod^1}  =  \Re^4$, respectively.
	
    \paragraph*{Rotational Equivariance}
	Next, we show that the translational dynamics is equivariant with respect to the one-dimensional rotations similar to those introduced in the monolithic model. 
    More explicitly, the group action $G$ on $\mathcal{S}_{\mod^1}$ and $\mathcal{A}_{\mod^1}$ is defined by
	\begin{align}
		gs_{\mod^1} & = (\rho_\theta(g) x, \rho_\theta(g) v, \rho_\theta(g) b_3, \rho_\theta(g) \omega_{12}), \label{eqn:gs_module1} \\
		ga_{\mod^1}  & = (f, \rho_\theta(g) \tau), \label{eqn:ga_module1}
	\end{align}
    where $\rho_\theta(g)$ is identical to \eqref{eqn:rho_theta_g}.
	
	\begin{prop}\label{prop:sym_module1}
		The translational dynamics, given by \eqref{eqn:decoup_x_dot}--\eqref{eqn:decoup_w12_dot}, is equivariant.
	\end{prop}
	
	\begin{proof}
	The equations of motion given by \eqref{eqn:decoup_x_dot}--\eqref{eqn:decoup_w12_dot} can be rearranged into $\dot s_{\mod^1} = F(s_{\mod^1}, a_{\mod^1})$ for an appropriate $F$.
		Under the group action $g$, $F\circ g$ is 
		\begin{align*}
            F( g s_{\mod^1}, g a_{\mod^1} )  = 
			\begin{bmatrix} 
				\rho_\theta(g) v \\
				m\grav e_3 - f \rho_\theta(g) R e_3 \\
				\rho_\theta(g) \omega_{12} \times \rho_\theta(g) b_3 \\
				\rho_\theta(g) \tau
			\end{bmatrix}.
		\end{align*}
		Since $R x \times R y = R (x \times y)$ for any $x, y \in \Re^3$ and $R \in \SO$, and $\rho_\theta(g)\in\SO$, we can show that the above reduces to $F \circ g = (I_{4 \times 4} \otimes \rho_\theta(g))F = g \circ F$, where $\otimes $ denotes the Kronecker product. 
	\end{proof}

	\subsubsection{Yaw Dynamics Module}
	
	From \eqref{eqn:R_dot} and \eqref{eqn:W_dot}, the remaining one-dimensional yaw dynamics, decoupled from the roll/pitch dynamics, can be written as
	\begin{align}
		\dot b_1 &= \omega_{23}\times b_1 = - \Omega_2 b_3 + \Omega_3 b_2,\label{eqn:decoup_b1_dot}\\
		J_3 \dot \Omega_3 &= M_3.\label{eqn:decoup_W3_dot}
	\end{align}
	where $\omega_{23} = \Omega_2 b_2 + \Omega_3 b_3 \in \Re^3$ represents the angular velocity for yawing.

	Let $b_{1_d}(t) \in \Sph^2$ be a desired direction of the first body-fixed axis, given by a smooth path on $\Sph^2$. 
    Here, we reformulate the above yaw dynamics into the error dynamics representing the difference between $b_1$ and $b_{1_d}$. 
    In general, we cannot guarantee that $b_1$ asymptotically converges to $b_{1_d}$ as $b_1$ is always normal to $b_3$, but the desired trajectory $b_{1_d}$ does not necessarily satisfy the same constraint. 
	To resolve this, we project $b_{1_d}$ onto the plane perpendicular to $b_3$ to define a projected yaw command $b_{1_c} = (I_{3\times 3}-b_3 b_3^T ) b_{1_d} \in \Sph^2$, ensuring that $b_{1_c} \perp b_3$ always.
    Then, the yawing error $e_{b_1} \in (-\pi,\pi]$ is formulated to be the angle between $b_1$ and $b_{1_c}$ within the plane normal to $b_3$, i.e.,
	\begin{align}
		e_{b_1} = \mathrm{atan2} (-b_{1_c} \cdot b_2,~ b_{1_c} \cdot b_1). \label{eqn:eb1}
	\end{align}
	Note that, despite $b_1$ being a two-dimensional unit vector in $\Sph^2$, the yaw error is reduced to be one-dimensional on $\Sph^1$ due to the constraint that both $b_1$ and $b_{1_c}$ are orthogonal to $b_3$.

    One can show that the time derivative of $e_{b_1}$ is given by
	\begin{align}
		\dot e_{b_1} &=\Omega_3 - \omega_{c_3}, \label{eqn:eb1_dot}
    \end{align}
	where $\omega_{c_3} = b_3 \cdot \omega_c \in \Re$ with $\omega_c\in\Re^3$ representing the angular velocity of $b_{1_c}$ resolved in the inertial frame, i.e., $\dot b_{1_c} = \omega_c\times b_{1_c}$.

    By taking the time derivative of this and substituting \eqref{eqn:decoup_W3_dot}, the yawing error dynamics is simplified into
	\begin{align}
		\ddot{e}_{b_1} &= \frac{1}{J_3} M_3 - \dot{\omega}_{c_3}. \label{eqn:eb1_2dot}
	\end{align}
	As such, the state and action variables for the yawing module are $s_{\mod^2} = (e_{b_1}, \dot{e}_{b_1}) \in \mathcal{S}_{\mod^2} = \Re^2$ and $a_{\mod^2} = M_3 \in \mathcal{A}_{\mod^2}  =  \Re$, respectively.
	
    \paragraph*{Reflective Equivariance}
	Next, we show that the above yaw error dynamics exhibits an inherent symmetry with respect to the finite cyclic subgroup $G = \{ 1, -1 \}$ of $\ensuremath{\mathsf{GL(1)}}$ generated by $-1$.
	This symmetry ensures the equivariance of the system with respect to the simultaneous flipping of signs for the concatenated pair $(s_{\mod^2}, a_{\mod^2}) = (e_{b_1}, \dot{e}_{b_1}, M_3)$ as illustrated in \Cref{fig:ref_sym_concept}.
	
	Let $\rho_r$ be a group representation acting on a scalar $x \in \Re$ for $g \in G$, defined as $\rho_r(g)x = x$ if $g = 1$, and $\rho_r(g)x = -x$ if $g = -1$.
	This representation describes how scalar state or control variables transform under the action of the reflection symmetry group $G$.
	In particular, the action of $\rho_r(-1)$ corresponds to flipping the sign of the variable.
	The group $G$ acts on $(s_{\mod^2}, a_{\mod^2})$ via the representation $\rho_r$ as
	\begin{align}
		gs_{\mod^2} & = (\rho_r(g) e_{b_1}, \rho_r(g) \dot{e}_{b_1}), \label{eqn:gs_module2} \\
		ga_{\mod^2}  & = (\rho_r(g) M_3). \label{eqn:ga_module2}
	\end{align}
	
	\begin{prop}\label{prop:sym_module2}
		The yaw error dynamics \eqref{eqn:eb1_2dot} is equivariant if $|\dot{\omega}_{c_3}| \ll 1$.
	\end{prop}
	
	\begin{proof}
        Let $(s_1, s_2) = (e_{b_1}, \dot e_{b_1})\in\Re^2$ be the state variable of the yawing dynamics. The second-order dynamics of \eqref{eqn:eb1_2dot} is rearranged into
        \begin{align*}
            \dot s_1 & = s_2,\\
            \dot s_2 & = \frac{1}{J_3}M_3 -\dot\omega_{c_3},
        \end{align*}
        where the right hand side is consolidated into $F(s,a)\in\mathcal{S}_{\mod^2}\times\mathcal{A}_{\mod^2}\rightarrow\Re^2$. 
        It is straightforward to show that $F(-s,-a)=-F(s,a)$ if $\dot\omega_{c_3}=0$.
	\end{proof}
    In summary, in the modular model of the quadrotor dynamics, the first module for the translational dynamics satisfies the rotational equivariance, and the second module for the yawing error dynamics satisfies the reflective equivariance. 
	
	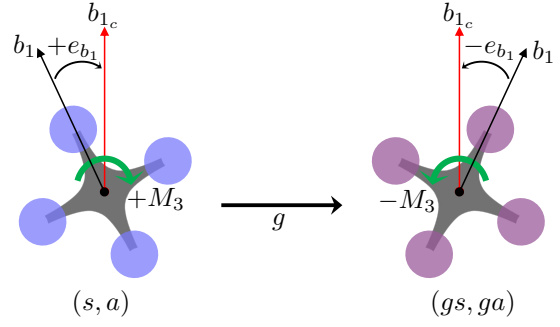
\begin{figure}[t]
		\centering
		\vspace*{0.3cm}
		\begin{tikzpicture}
			\node[anchor=south west, inner sep=0] (image) at (0, 0) {\includegraphics[width=0.4\textwidth]{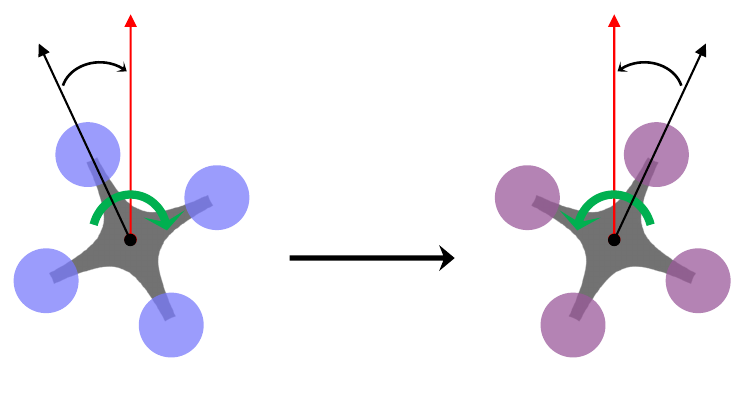}};
			\begin{scope}[x={(image.south east)}, y={(image.north west)}]
				\node at (0.03, 0.9) {$b_1$}; 
				\node at (0.175, 1.) {$b_{1_c}$}; 
				\node at (0.12, 0.89) {$+e_{b_1}$};
				\node at (0.27, 0.395) {$+M_3$};  
				\node at (0.17, 0.05) {$(s,a)$}; 
				\node at (0.495, 0.32) {$g$}; 
				\node at (0.85, 0.05) {$(gs,ga)$}; 
				\node at (0.98, 0.88) {$b_1$}; 
				\node at (0.825, 1.) {$b_{1_c}$}; 
				\node at (0.882, 0.89) {$-e_{b_1}$};
				\node at (0.73, 0.39) {$-M_3$};  
			\end{scope}
		\end{tikzpicture}
		\vspace*{-0.1cm}
		\caption{Illustration of the group action corresponding to the reflection symmetry with respect to the finite cyclic subgroup $G = \{ 1, -1 \}$. 
			The reflectional action $g \in G$ transforms the original state-action pair $(s, a)$ (left) into the reflected pair $g(s, a)$ (right), preserving the symmetry of the yaw error dynamics.}
		\label{fig:ref_sym_concept}
	\end{figure}

	\section{Equivariant Reinforcement Learning for Quadrotor Dynamics}\label{sec:EquivRL}	

    In the preceding section, we identified the equivariance properties of the quadrotor dynamics, and we showed that a set of trajectories satisfying the dynamics can be generated by applying the group action to a single trajectory. 
    Or, such a set of trajectories related to the group action can be identified by an equivalent class. 
    In reinforcement learning that requires an extensive set of multiple trajectories, the former can be utilized to diversify the training data with the group action, and the latter can be encoded for generalization capabilities, where the outcome of learning for a single trajectory is extended to other trajectories related by the group action. 

	Equivariant reinforcement learning is to exploit these properties to enhance sampling efficiency and generalization abilities. 
    This section shows that the equivariance properties of the dynamics result in symmetry properties of the optimal value function and the optimal policy in reinforcement learning, which can be implemented via equivariant neural networks.
    And we illustrate how these can be applied to the quadrotor dynamics.

    \subsection{Equivariant Reinforcement Learning}

	Here, we formulate the equivariance reinforcement learning for the continuous-time Markov decision process.
    Suppose that the state evolves according to the equation of motion $\dot s = F(s,a)$ for $F:\mathcal{S}\times\mathcal{A}\rightarrow T\mathcal{S}$, and the reward is given by $r:\mathcal{S}\times\mathcal{A}\rightarrow\Re$.  
	For $0<\gamma<1$, let the value function of a policy $\pi:\mathcal{S}\rightarrow\mathcal{A}$ be defined as
	\begin{align}
		V_\pi (t,s(t)) = \int_t^\infty \gamma^{\tau-t} r(s(\tau),a(\tau)) d\tau, \label{eqn:V}
	\end{align}
    where the trajectory $(s(\tau),a(\tau))$ for $\tau\in[t,\infty)$ is sampled by the policy $\pi$ and the equation of motion. 
	The objective is to identify the optimal policy $\pi^*(s(t))$ that maximizes $V_\pi(t,s(t))$.

	We show that when the dynamics and the reward function are symmetric with respect to a group action, the resulting value function and the optimal policy satisfy the following properties. 
    \begin{prop}\label{prop:ERL}
	Consider a continuous-time MDP to maximize the value function \eqref{eqn:V}.
	Suppose that $F:\mathcal{S}\times\mathcal{A}\rightarrow T\mathcal{S}$ is equivariant and the reward $r:\mathcal{S}\times\mathcal{A}\rightarrow\Re$ is invariant with respect to a group action $g:\mathcal{S}\times\mathcal{A}\rightarrow\mathcal{S}\times\mathcal{A}$, i.e., 
	\begin{gather}
		F\circ g = g\circ F, \label{eqn:F_eqv} \\
		r\circ g = r. \label{eqn:r_inv}
	\end{gather}
	Then, the following properties hold:
	\renewcommand{\theenumi}{(\roman{enumi}}
	\begin{enumerate}
		\item The optimal value function is G-invariant under the group action, i.e., $V^*_\pi = V^*_\pi\circ g$.
		\item The optimal policy is G-equivariant under the group action, i.e., $\pi^* \circ g = g\circ \pi^*$. 
	\end{enumerate}
	\end{prop}
    \begin{proof}
        The value function presented in \eqref{eqn:V} is defined as a function of the time $t$, and the state $s(t)$ at that time $t$, where the state and action trajectory afterward is determined by the given policy $\pi$.
        Instead of relying on the policy to compute the trajectory starting from the given state, we relax the definition of $V$ as a function of the state and action trajectory as follows. 
        Let $\phi = (s(t), a(t)):[t,\infty)\rightarrow\mathcal{S}\times\mathcal{A}$ be a trajectory over the interval $[t,\infty)$ starting from $s(t)=s_0\in\mathcal{S}$. 
        The value of the trajectory $V(\phi)$ is computed by the right hand side of \eqref{eqn:V}. 

		According to \Cref{prop:sym_mono}, $\tilde\phi = (\tilde s(t), \tilde a(t))= g\phi = (gs(t), ga(t))$ is another trajectory of the system, and its value is
		\begin{align}
            V(\tilde\phi) & = \int_t^\infty \gamma^{\tau-t}\, r(\tilde s(\tau),\tilde a(\tau)) d\tau \nonumber\\
                          & = \int_t^\infty \gamma^{\tau-t} \times (g\circ r)(s(\tau),a(\tau)) d\tau,\label{eqn:prop4_1}
		\end{align}
		which reduces to $V(\phi)$ due to the invariance of the reward given by \eqref{eqn:r_inv}.
        Thus, $V(\phi) = V(\tilde\phi)$, which implies that the state-action trajectory transformed by the group action $g$ has the same value as the original trajectory for any $g\in G$.

        Next, let $\phi^*= (s^*(t), a^*(t))$ be the optimal state and action trajectory starting from $s_0$ at $t$, driven by the optimal policy $\pi^*$. 
        Also, let $\tilde \phi^* = g\phi^*$ be the optimal trajectory transformed by a group element $g\in G$. 
        From the above, they have the same value, i.e., 
        \begin{align}
            V(\phi^*) = V(\tilde \phi^*).\label{eqn:prop4_0}
        \end{align}

        Now, we show that $\tilde\phi^* = g\phi^*$ is another optimal trajectory starting from $gs_0$ by contradiction. 
        Suppose $\tilde\phi^*$ is not optimal. 
        Then, there exists another trajectory $\tilde\phi'=(\tilde s'(t), \tilde a'(t))$ starting from the same state $\tilde s'(t)= gs_0$ at $\tilde\phi^*$ but with a larger value, i.e.,
        \begin{align*}
            V(\tilde \phi^*) < V(\tilde \phi').
        \end{align*}
        We transform both trajectories $\tilde\phi^*$ and $\tilde\phi'$ by $g^{-1}$ to obtain $\phi^*$ and $g^{-1}\tilde \phi'$, respectively.  
        Since the value does not change by any transformation, the above inequality still holds after being transformed by $g^{-1}$, i.e.,
        \begin{align*}
            V(\phi^*) < V(g^{-1}\tilde \phi').
        \end{align*}
        The initial state of $g^{-1}\tilde\phi'$ is given by $g^{-1}\tilde s'(t) = g^{-1} g s_0 = s_0$.
        This implies that there exists another state action trajectory, namely $g^{-1}\tilde\phi'$ starting from $s_0$ with a value larger than the optimal value $V(\phi^*)$, which contradicts the fact that $\phi^*$ is the optimal path. 
        Therefore, the transformed trajectory $\tilde\phi^* = g\phi^*$ is the optimal trajectory starting from $gs_0$, and \eqref{eqn:prop4_0} shows (i).

        Since both of $\phi^*$ and $\tilde\phi^*=g\phi^*$ are generated by the optimal policy $\pi^*$, we have $a^* = \pi^*(s)$ and $\tilde a^* = \pi^*(\tilde s)$. 
        By combining these,
        \begin{align*}
            \pi^*(s) = a^* = g^{-1}\tilde a^* = g^{-1} \pi^*( \tilde s ) = g^{-1}\pi^*(g s).
        \end{align*}
        Taking the group action to both sides yields (ii).
	\end{proof}

	\begin{remark}
	    While \Cref{prop:ERL} is developed for the continuous-time system, the invariance of the value function is applied to the state-action value function that is commonly used in the reinforcement learning of discrete-time systems. 
	Let the time be discretized by the sequence $\{t_0, t_1,\ldots\}$. The state-action value function at $t_k$ is given by
	\begin{align*}
        Q(s_k, a_k) = \int_{t_k}^{t_{k+1}} \gamma^{\tau-t_k} r(s(\tau), a_k) d\tau + V^*(t_{k+1}, s(t_{k+1})),
	\end{align*}
	where the state trajectory $s(\tau)$ over the interval $[t_k, t_{k+1}]$ is governed by $\dot s = F(s, a_k)$ with the boundary condition $s(t_k)=s_k$, for a fixed $a_k\in\mathcal{A}$. 
	From the discussion following \eqref{eqn:prop4_1}, the first term of the right hand side is $G$-invariant, and so is the second term from \Cref{prop:ERL}.
	Thus, the state-action value function is also $G$-invariant. 
	\end{remark}
	
	The invariance of the value function and the equivariance of the policy presented in \Cref{prop:ERL} are particularly useful for improving the sampling efficiency and generalization capability of reinforcement learning. The objective of the common actor-critic framework is to learn the value function and the policy from sampled data, where both the quantity and quality of the data play an important role in training.
	\Cref{prop:ERL} implies that, under the given assumption, the value $V(s)$ learned at a particular state $s$ is automatically generalized to the value $V(gs)$ at $gs$ for any $g \in G$. Similarly, the optimal action $\pi(s)$ at the state $s$ can be used to compute the optimal action $\pi(gs) = g\pi(s)$ at another state $gs$ for any $g \in G$.
	This effectively reduces the domain of learning from $\mathcal{S} \times \mathcal{A}$ to the quotient space $\mathcal{S} \times \mathcal{A} / \sim$.

	For implementation in deep reinforcement learning, the value function and the policy can be modeled by an equivariant neural network~\cite{finzi2021practical}, which is a type of neural network that is specifically designed to respect the above symmetries in the input data and the output. 
	
	Specifically, assume that the state and the action are embedded in the Euclidean space, such that the group action on the state and action is expressed in terms of the representation, namely $\rho_s(g)$ and $\rho_a(g)$, respectively. 
	The $G$-invariant value function and the state-action value function can be modeled by an equivariant neural network satisfying
	\begin{align} 
	    V(s) = V(\rho_s(g)s), \quad
	    Q(s, a) = Q(\rho_s(g)s, \rho_a(g)a). \label{eqn:value_equiv}
	\end{align}	
	And the $G$-equivariant actor network for the policy is modeled by a neural network satisfying
	\begin{align} 
	    \pi (\rho_s(g)s) = \rho_a(g) \pi(s). \label{eqn:pi_equiv}
	\end{align}	
	Then, any deep reinforcement learning technique can be applied to the above neural network architecture without any further modification required. 
	In other words, the presented equivariant reinforcement learning framework is readily integrated with any other reinforcement approaches.

	\subsection{Equivariant Monolithic RL for Quadrotor UAV}
	
	Here we formulate the equivariant RL for the quadrotor dynamics using the symmetry properties identified in \Cref{sec:SymDyn}. 
	To illustrate the implication of the equivariant RL, we consider the planar motion of quadrotors, as shown in \Cref{fig:equivAC}, where the plane is normal to the gravity. 
	Consider a quadrotor at a state $s\in\mathcal{S}$, and let $a\in\mathcal{A}$ be the optimal action at the state $s$. 
	If the quadrotor is rotated by an angle $\theta$ according to the symmetry discovered at \Cref{prop:sym_mono}, into a new state $\rho_s(g) s$, then the corresponding optimal action at the new state is given by $\rho_a(g) a$ according to \eqref{eqn:pi_equiv}, which is the rotation of $a$ by the same angle $\theta$.
	In other words, rotating the quadrotor yields the rotation of the optimal action.
	Or equivalently, the diagram at \Cref{fig:equivAC}.(a) commutes. 
	Next, for these rotations of the state and the action, the state value or the state-action value remain unchanged from \eqref{eqn:value_equiv}.
	
	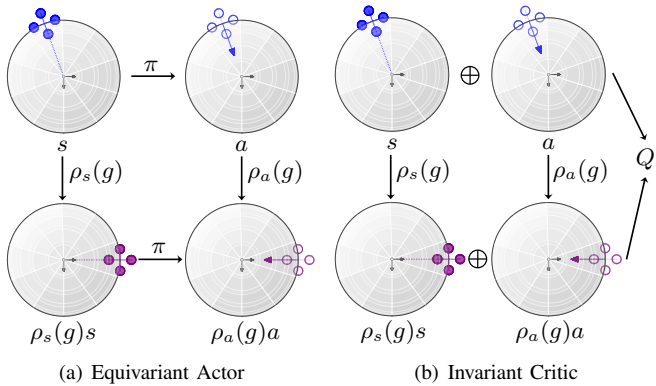
\begin{figure}[b]
		\centering
		\begin{tabular}{cc}
			\subfigure[Equivariant Actor]{\includegraphics[width=0.47\columnwidth]{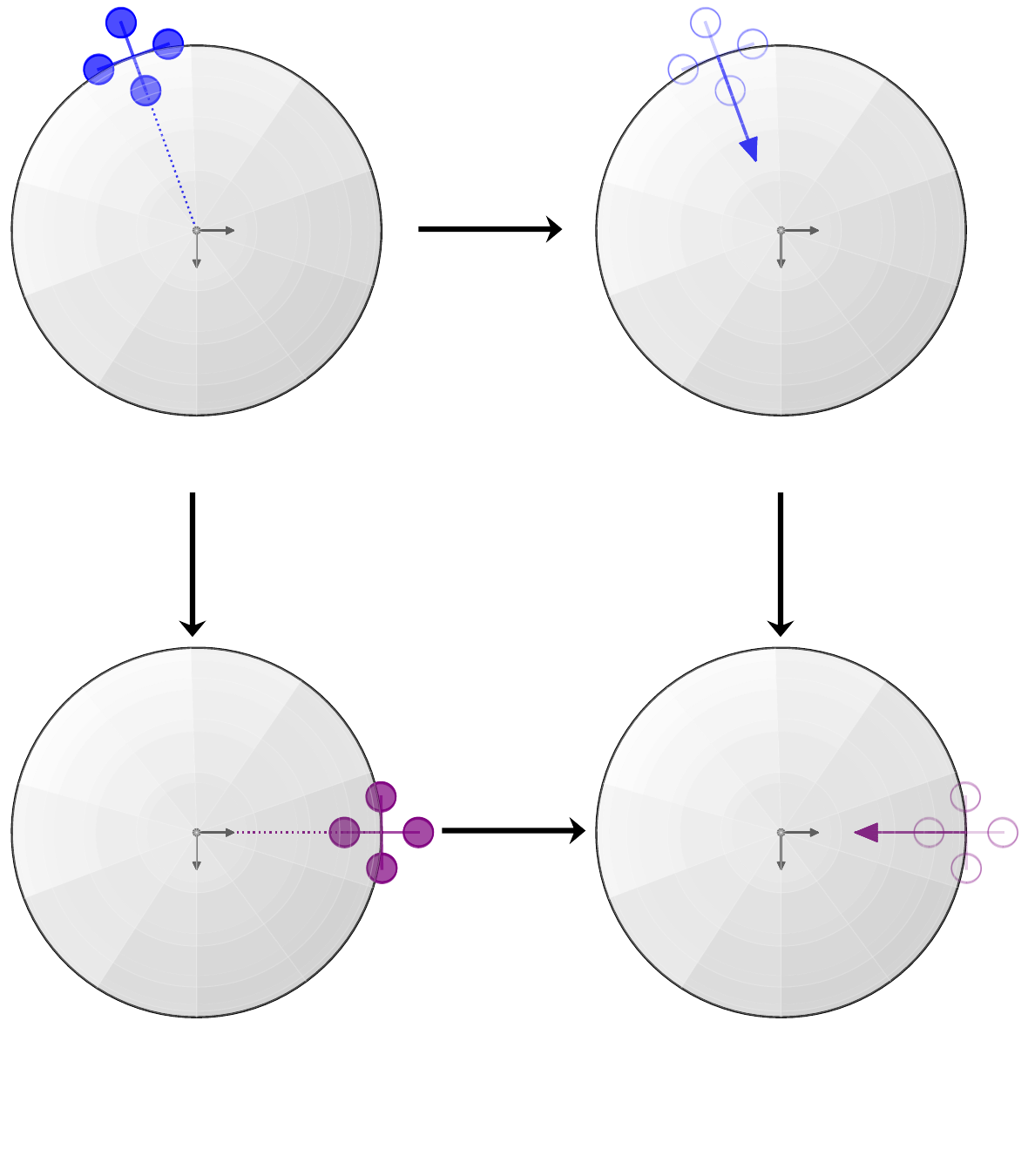} \label{fig:equivActor} \hspace{-0.47cm}} 
			& \subfigure[Invariant Critic]{\includegraphics[width=0.48\columnwidth]{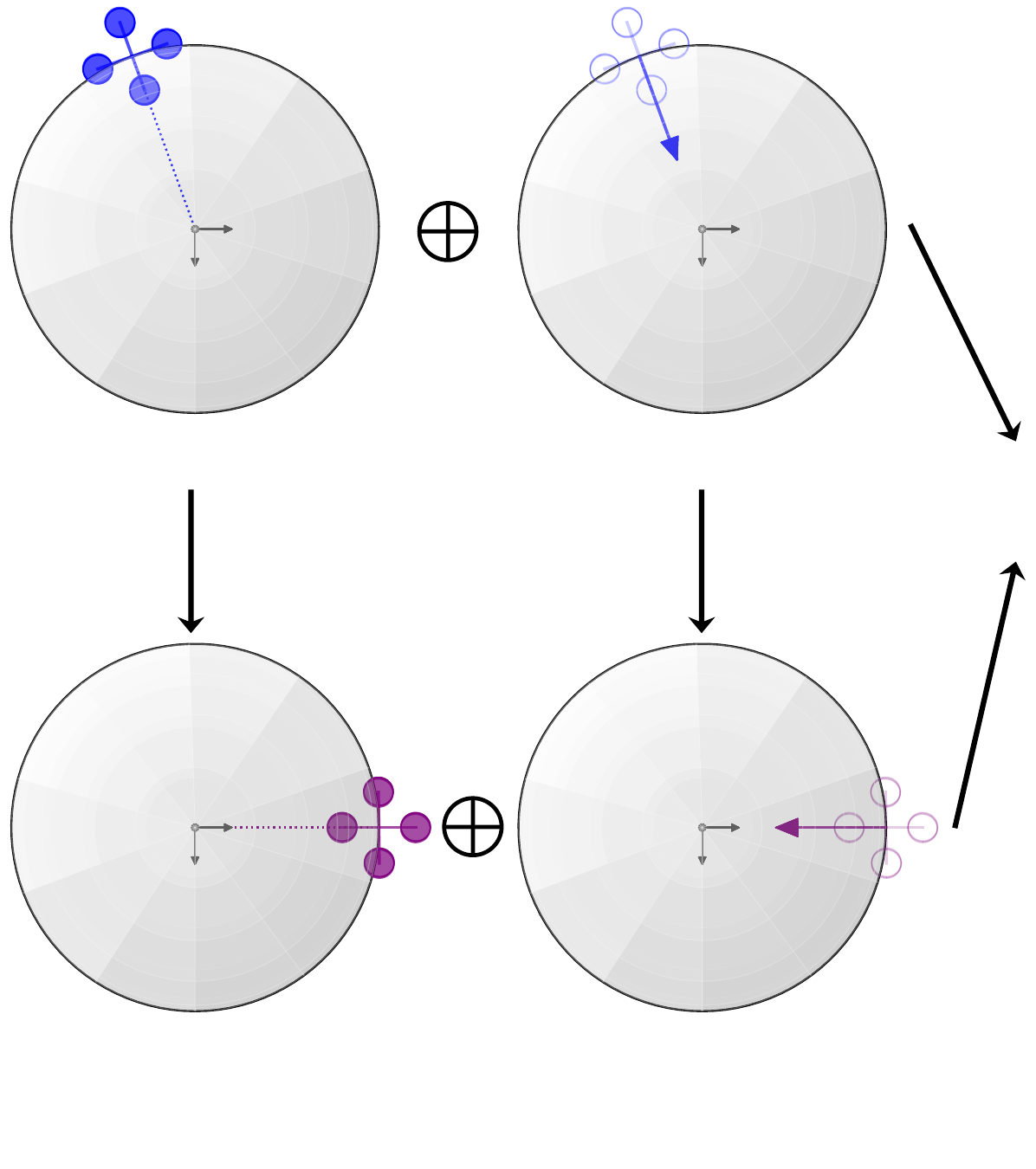} \label{fig:equivCritic}}
		\end{tabular}
		\caption{Illustration of equivariant actors and invariant critics.}
		\label{fig:equivAC}
        \small
		\begin{tikzpicture}[overlay, remember picture]
			\node at (-3.38, 4.32) {$s$}; 
			\node at (-1., 4.32) {$a$}; 
			\node at (-2.2, 5.38) {$\pi$}; 
			\node at (-2.93, 3.97) {$\rho_s(g)$}; 
			\node at (-0.55, 3.97) {$\rho_a(g)$}; 
			\node at (-3.37, 1.83) {$\rho_s(g)s$}; 
			\node at (-0.95, 1.83) {$\rho_a(g)a$}; 
			\node at (-2.1, 2.96) {$\pi$}; 
			
			\node at (1.0, 4.35) {$s$}; 
			\node at (3.08, 4.35) {$a$}; 
			\node at (1.42, 3.99) {$\rho_s(g)$}; 
			\node at (3.51, 3.99) {$\rho_a(g)$}; 
			\node at (1.02, 1.84) {$\rho_s(g)s$}; 
			\node at (3.1, 1.84) {$\rho_a(g)a$}; 
			\node at (4.37, 4.14) {$Q$}; 
		\end{tikzpicture}
	\end{figure}	

	First, we develop a monolithic RL framework that respects these equivariance properties.
	The control objective is to minimize the tracking errors for a set of goal states defined by the desired position $x_d\in\Re^3$, desired velocity $v_d\in\Re^3$, desired heading direction $b_{1_d}\in\Sph^2$, and desired angular velocity $\Omega_d\in\Re^3$.
	Here we assume that the goal states are fixed as $x_d = v_d = [0, 0, 0]^T$, while $\Omega_d = [0, 0, \omega_{c_3}]^T$, where the desired angular velocity for yawing is set as $\Omega_{d_3} = \omega_{c_3}$ according to \eqref{eqn:eb1_dot}.	
	In the training process, we define the observation space using error terms, 
	\begin{align*} 
		o_\mono = (e_x, e_{I_x}, e_v, R, e_{b_1}, e_{I_{b_1}}, e_\Omega) \in \Re^{14} \times \SO,
	\end{align*} 
	where the tracking errors are defined as $e_x = x - x_d\in\Re^3$, $e_v = v - v_d\in\Re^3$, and $e_\Omega = \Omega - \Omega_d\in\Re^3$.
	Also, the position and yaw integral terms, $e_{I_x}\in\Re^3$ and $e_{I_{b_1}}\in\Re$, are introduced to mitigate steady-state errors.
	These terms are updated according to
	\begin{align} 
		\dot e_{I_x} = -\alpha e_{I_x} + e_x, \quad
        \dot e_{I_{b_1}} = -\beta e_{I_{b_1}} + e_{b_1}.\label{eqn:eI_dot}
	\end{align}
	where $\alpha, \beta > 0$ prevent the potential overflow of $e_{I_x}$ and $e_{I_{b_1}}$ due to noise or delays during real-world deployment.
	
    As discussed at \Cref{sec:dyn_mono}, the monolithic model of the quadrotor dynamics is equivariant under the group action that corresponds to the rotation about the vertical direction. 
    The additional state variables for the integrated errors governed by \eqref{eqn:eI_dot} also satisfy the equivariance.
    Specifically, the group action on the state given by \eqref{eqn:gs_mono} is extended to the observation variable, and it is described by the following representation:
	\begin{align} 
		 \rho_{o_\mono}(g) = 4 \rho_\theta(g) \oplus 3 \rho_e(g), \label{eqn:o_mono_rep}
	\end{align} 
	where $\oplus$ corresponds to the direct sum, and the shorthand $c \rho(g)$ for a positive integer $c$ denotes $\rho(g) \oplus \rho(g) \oplus \ldots \oplus \rho(g)$, where the direct sum is repeated $c$ times.
    Here $\rho_\theta(g)$ is introduced at \eqref{eqn:rho_theta_g} and $\rho_e(g)$ corresponds to the identity matrix of an appropriate dimension. 
	More explicitly, in \eqref{eqn:o_mono_rep}, each $\rho_\theta(g)$ acts on $(e_x, e_{I_x}, e_v, R)$, while each $\rho_e(g)$ acts on $(e_{b_1}, e_{I_{b_1}}, e_\Omega)$.

    Next, the representation of the group action \eqref{eqn:ga_mono} on the action variables,  $a_\mono = (f, M)$, is given by
	\begin{align} 
        \rho_{a_\mono}(g) = \rho_e(g) \oplus \rho_e(g).\label{eqn:a_mono_rep}
	\end{align} 
	In other words, $a_\mono$ remains unchanged by the group. 
	
	Finally, the reward function is designed as 
	\begin{align*}
		r_\mono & = -k_x \|e_x\|^2 -k_{I_x} \|e_{I_x}\|^2 -k_{b_1} |e_{b_1}| -k_{I_{b_1}} |e_{I_{b_1}}|^2  \nonumber\\
		& \quad -k_v \|e_v\|^2-k_\Omega \|e_\Omega\|^2 - r_\textrm{crash},
	\end{align*}
	where each $k \in\Re$ are positive weighting constants.
	The first four terms prioritize minimizing position and yaw errors, the next two terms penalize overly aggressive maneuvers, and the final term $r_\textrm{crash} \in\Re$ imposes a large penalty to avoid collisions and out-of-bounds flight.
	Importantly, it is straightforward to show the reward remains invariant under rotations, as taking the norms of vectors or the absolute values of scalars is unaffected by any rotation.

    Then, from \Cref{prop:ERL}, the optimal policy is $G$-equivariant and the value is $G$-invariant. 
    In deep RL, these symmetries are enforced by utilizing equivariant neural networks~\cite{finzi2021practical} to model the policy and the value. 
    Specifically, the equivariant neural networks provide a multilayer perception that respects the equivariance or the invariance for given representations \eqref{eqn:o_mono_rep} and \eqref{eqn:a_mono_rep} of the group action.  
    Once the policy and the value are modeled by the equivariant neural network, any deep RL approach can be applied. 
	
    \subsection{Equivariant Modular RL for Quadrotor UAV}
    This section extends the equivariant RL to the modular framework introduced in \Cref{sec:dyn_mod}, where the quadrotor dynamics are decomposed into the translational module and the yawing module.
	By modularizing the RL structure, each module is focused on specific control objectives, facilitating efficient learning.
	
	\subsubsection{Translational Motion RL Module}
	This module is designed to handle the translational motion of the quadrotor and is developed to be equivariant for rotations about the vertical axis. 
	The observation space is defined as
	\begin{align*} 
	o_{\mod^1} = (e_x, e_{I_x}, e_v, b_3, e_{\omega_{12}}) \in \Re^{12}\times \Sph^2,
	\end{align*}	
	where $e_{\omega_{12}} = \omega_{12} - \omega_{d_{12}} \in \Re^3$ is the tracking error for the angular velocity of $b_3$.
	Similar with the above, we assume that $x_d = v_d = \omega_{d_{12}} = [0, 0, 0]^T$ during training.
	The group action on the observation state is represented by
	\begin{align} 
        \rho_{o_{\mod^1}}(g) = 5 \rho_\theta(g),\label{eqn:rho_mod_1}
	\end{align}	
	where each $\rho_\theta(g)$ acts on $(e_x, e_{I_x}, e_v, b_3, e_{\omega_{12}})$, i.e., each state is rotated by $\theta$.
	Similarly, for $a_{\mod^1} = (f, \tau)$, the representation of the group action is
	\begin{align} 
        \rho_{a_{\mod^1}}(g) = \rho_e(g) \oplus \rho_\theta(g),\label{eqn:rho_a_mod_1}
	\end{align}	
	where the control thrust $f$ remains invariant, while the torque $\tau$ that is resolved in the inertial frame is rotated. 
	
	Lastly, to minimize the translational tracking errors, the reward function is structured as
	\begin{align*}
		r_{\mod^1} = & -k_x \|e_x\|^2 -k_{I_x} \|e_{I_x}\|^2 -k_v \|e_v\|^2  -k_{\omega_{12}} \|e_{\omega_{12}} \|^2  \nonumber\\
		& - r_\textrm{crash},
	\end{align*}
	where $k\in\Re$ are positive weighting constants.
	The first two terms minimize position and steady-state errors, while the latter two terms suppress aggressive maneuvers.
    It is straightforward to show that the reward is invariant with respect to the group action of \eqref{eqn:rho_mod_1}.
    Then, the results of \Cref{prop:ERL} are applied, and any deep RL approach can be utilized with the equivariant neural network constructed by the representations \eqref{eqn:rho_mod_1} and \eqref{eqn:rho_a_mod_1}.
	
	\subsubsection{Yaw RL Module}
	The yaw RL module is responsible for aligning the quadrotor's first body-fixed axis $b_1$ with the desired direction $b_{1_c}$ by controlling the yaw moment $a_{\mod^2} = M_3$.
	For the yaw module, the observation space is defined as 
	\begin{align}
		o_{\mod^2} = (e_{b_1}, e_{I_{b_1}}, \dot{e}_{b_1}) \in \mathbb{R}^3,
	\end{align}
	where $e_{I_{b_1}}$ is the integral of $e_{b_1}$ to mitigate the yaw steady-state error.
	
	Given the reflection symmetry of the yaw error dynamics under the cyclic group $G = \{1, -1\}$, the representations for the group action for the observation and the action are given by
	\begin{align}
		\rho_{o_{\mod^2}}(g) = 3 \rho_r(g), \quad
		\rho_{a_{\mod^2}}(g) = \rho_r(g),
	\end{align}
    where $\rho_r(g)\in\{1,-1\}$. 
    According to \Cref{prop:sym_module2}, this symmetry implies that the yaw dynamics remains equivariant under the simultaneous sign flips of $(e_{b_1}, e_{I_{b_1}}, \dot{e}_{b_1}, M_3)$.
	
	Finally, the reward function is designed to penalize yaw tracking errors as
	\begin{align*}
		r_{\text{mod}^2} = -k_{b_1} |e_{b_1}| - k_{I_{b_1}} |e_{I_{b_1}}|^2 - k_{\Omega_3} |\dot{e}_{b_1}|^2 - r_\text{crash},
	\end{align*}
    which is invariant under the group action. 
    Similar with the above, the equivariant RL can be applied to the yaw module.

	\section{Experiments}\label{sec:Exp}	
	In this section, we present learning results for the proposed equivariant RL for quadrotors, where the presented monolithic and modular equivariant RL frameworks are compared with their non-equivariant counterparts.
    Specifically, we benchmark four distinct frameworks: two traditional non-equivariant approaches using standard multi-layer perceptrons (MLPs), namely \textbf{Mono-MLP} (Monolithic Policy with MLP) and \textbf{Mod-MLP} (Modular Policies with MLP), and two proposed equivariant approaches utilizing equivariant multilayer perceptrons (EMLPs), referred to as \textbf{Mono-EMLP} (Monolithic Policy with EMLP) and \textbf{Mod-EMLP} (Modular Policies with EMLP).
	The flight performance of these frameworks is evaluated through simulation and real-world flight experiments.
	The source code is available at \url{https://github.com/fdcl-gwu/gym-rotor}. 
		
	\subsection{Training Details}
    This section details the implementation and training of the proposed RL framework, covering network structures, domain randomization techniques to mitigate the sim-to-real gap, policy regularization strategies for smooth control, and selected hyperparameters.

	The proposed equivariant RL framework is compatible with any single- or multi-agent RL algorithms. 
	Here, we implement the equivariant RL using three distinct RL techniques to evaluate its performance under varying RL formulations.
	Specifically, we employ Proximal Policy Optimization (PPO)\cite{schulman2017proximal}, an on-policy algorithm, alongside Twin Delayed Deep Deterministic Policy Gradient (TD3)\cite{fujimoto2018addressing} and Soft Actor-Critic (SAC)~\cite{haarnoja2018soft}, which represent deterministic and stochastic off-policy approaches, respectively.
	
	Furthermore, as shown in \cite{yu2024modular}, the modular framework supports two training strategies: (i) independent module training without information exchange, known as Decentralized Training and Decentralized Execution (DTDE), and (ii) cooperative training using multi-agent RL methods, termed Centralized Training and Decentralized Execution (CTDE). 
	In this work, we adopt the CTDE framework, as DTDE often suffers from non-stationarity due to the lack of inter-module communication, whereas CTDE’s centralized critic networks facilitate better coordination and higher rewards.

	During training, the initial states were randomly sampled from pre-defined distributions to encourage diverse exploration.
    Specifically, the quadrotor was sampled from a $1 \mathrm{m}^3$ box-shaped space at the beginning of each episode.
	Additionally, we normalize the reward signal to the range $[0, 1]$, and we scale the state and action spaces to the interval of $[-1, 1]$ at each training step for better computational properties. 
	Unlike existing approaches that often rely on auxiliary techniques or pre-training, all four frameworks are successfully trained through model-free learning.

    \paragraph*{Network Structures}
	The non-equivariant architectures, Mono-MLP and Mod-MLP, utilize standard MLPs for their critic and actor networks. 
	Their critic networks are designed as a two-layer MLP with 62 hidden nodes per layer, ensuring sufficient capacity to approximate the action-value function.
	The actor network architectures differ depending on modular or monolithic designs.
    Mono-MLP employs a single actor network with two layers, each containing 24 hidden nodes.
	In Mod-MLP, the actor network of the first module is designed as a two-layer MLP with 16 hidden nodes, while the second module's actor network adopts a two-layer MLP with 4 hidden nodes.
	
	Next, for the equivariant approaches, namely Mono-EMLP and Mod-EMLP, the equivariant actor and critic networks are implemented using the EMLP library~\cite{finzi2021practical}, 
	which provides a systematic framework for generating equivariant layers under specified group actions, ensuring these layers inherently respect the underlying symmetries of the data.
    To initialize EMLP layers, we specify the symmetry group, the input and output representations as discovered in \Cref{sec:EquivRL}, and the network architecture, described by the number of layers and the number of channels (feature dimension) per layer.
	The critic networks in both Mono-EMLP and Mod-EMLP frameworks consist of two EMLP layers, each with 62 equivariant channels. 
	For the actor network, Mono-EMLP employs two EMLP layers with 24 equivariant channels per layer, similar to Mono-MLP. 
	Lastly, the actor network within Mod-EMLP's first module is implemented as two EMLP layers with 16 equivariant channels, while the second module is designed as two EMLP layers with 4 equivariant channels.
	
	\begin{table}[b]
		\centering
		\setlength{\tabcolsep}{1.1em} 
		\renewcommand{\arraystretch}{1.1}  
		\caption{Quadrotor parameters}
		\label{tab:nominal_params}
		\begin{tabular}[t]{lc}
			\toprule
			Parameter & Nominal Value\\
			\midrule
			Mass, $m$ & \qty{2.15}{\kilogram} \\
			Arm length, $d$ & \qty{0.23}{\meter} \\
			Moment of inertia, $J$ & (0.022, 0.022, 0.035) \unit{\kilogram \meter ^2} \\
			Torque-to-thrust coefficient, $c_{\tau f}$ & 0.0135 \\
			Thrust-to-weight coefficients, $c_{tw}$ & 2.2 \\
			\bottomrule
		\end{tabular}
		\vspace*{-0.2cm}
	\end{table} 

	\paragraph*{Domain Randomization}  \label{sec:DM}
    Bridging the gap between simulation and real-world environments is a common challenge when deploying reinforcement learning (RL) agents trained in numerical simulations to real-world scenarios.
	To address this, domain randomization has been effectively applied~\cite{molchanov2019sim}. 
	This approach promotes the development of RL agents that exhibit adaptive and robust behaviors capable of generalizing across a wide range of conditions by randomizing properties (e.g., mass or arm length) during training.
	In this study, the simulator’s physical parameters, listed in~\Cref{tab:nominal_params}, are uniformly sampled within a range of $\pm 10\%$ of their nominal values at the start of each episode, thereby improving robustness to real-world variations.

	\begin{table}[t]
		\centering
		\setlength{\tabcolsep}{1.1em} 
		\renewcommand{\arraystretch}{1.1}  
		\caption{Hyperparameters used for RL training}
		\label{tab:hyper_params}
		\begin{tabular}{lc}
			\toprule
			\multicolumn{1}{c}{Parameter} & Value \\ 
			\midrule
			Optimizer & AdamW \\
			Discount factor, $\gamma$ & 0.99 \\
			Actor learning rate & $3 \cdot 10^{-4} \rightarrow 1 \cdot 10^{-5}$ \\
			Critic learning rate & $2 \cdot 10^{-4} \rightarrow 1 \cdot 10^{-5}$ \\ 
			Maximum global norm & 100 \\
			\hline
			\multicolumn{2}{l}{Soft Actor-Critic (SAC)} \\ 
			~~Replay buffer size & $10^6$ \\
			~~Batch size & 256 \\
			~~Target update interval & 3 \\
			~~Entropy regularization coefficient, $\alpha$ & 0.05 \\
			\hline
			\multicolumn{2}{l}{Twin Delayed DDPG (TD3)} \\ 
			~~Target smoothing coefficient, $\tau$ & 0.005 \\ 
			~~Exploration noise & 0.3 $ \rightarrow$ 0.05 \\
			~~Target policy noise & 0.2 \\
			~~Policy noise clip & 0.5 \\
			\hline
			\multicolumn{2}{l}{Proximal policy optimization (PPO)} \\ 
			~~Time horizon & 7000 \\
			~~Number of epochs & 20 \\
			~~Minibatch size & 128 \\
			~~Clipping ratio, $\epsilon$ & 0.2 \\
			~~GAE parameter, $\lambda$ & 0.9 \\
			~~Entropy coefficient & $1 \cdot 10^{-2}$ \\
			~~L2 regularization coefficient & $1 \cdot 10^{-4}$ \\
			\bottomrule
		\end{tabular}
	\end{table}

	\paragraph*{Smooth Control} \label{sec:smooth} 
    Another major challenge in deploying reinforcement learning (RL) is that transferred RL policies often generate physically unrealistic, high-frequency control signals. 
    These oscillatory motor signals can degrade performance and potentially damage hardware, leading to overheating and mechanical failure. 
    To address this issue, we incorporate regularization terms into the policy training process, including temporal, spatial, and magnitude regularization, as motivated by prior work~\cite{yu2024modular}.
	
	\begin{figure*}[t] 
		\centering
		\begin{tabular}{cccc}
			\hspace{-0.5cm}
			\subfigure[PPO]{\includegraphics[width=0.33\textwidth]{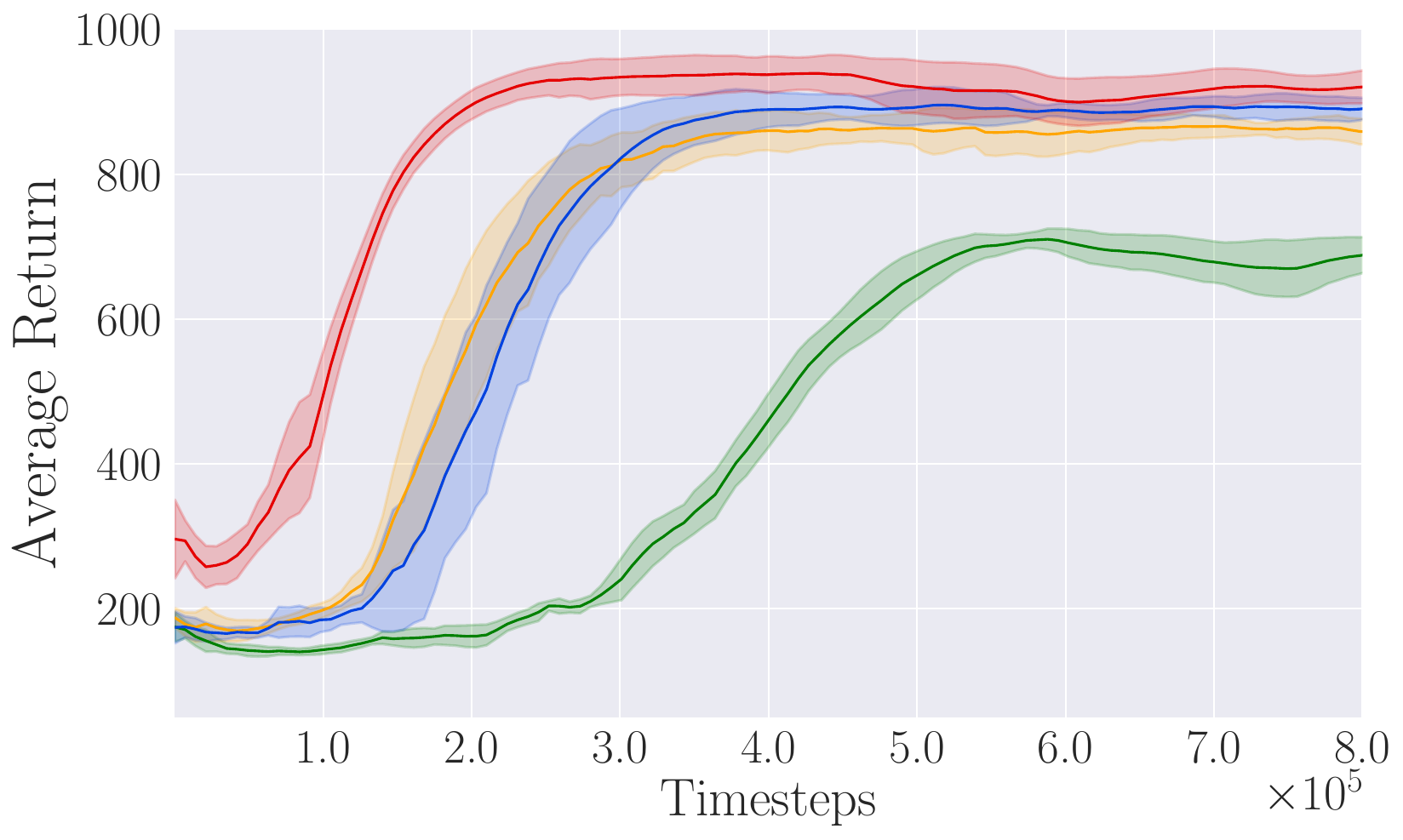} \label{fig:PPO}} \hspace{-0.45cm} &
			\subfigure[TD3]{\includegraphics[width=0.33\textwidth]{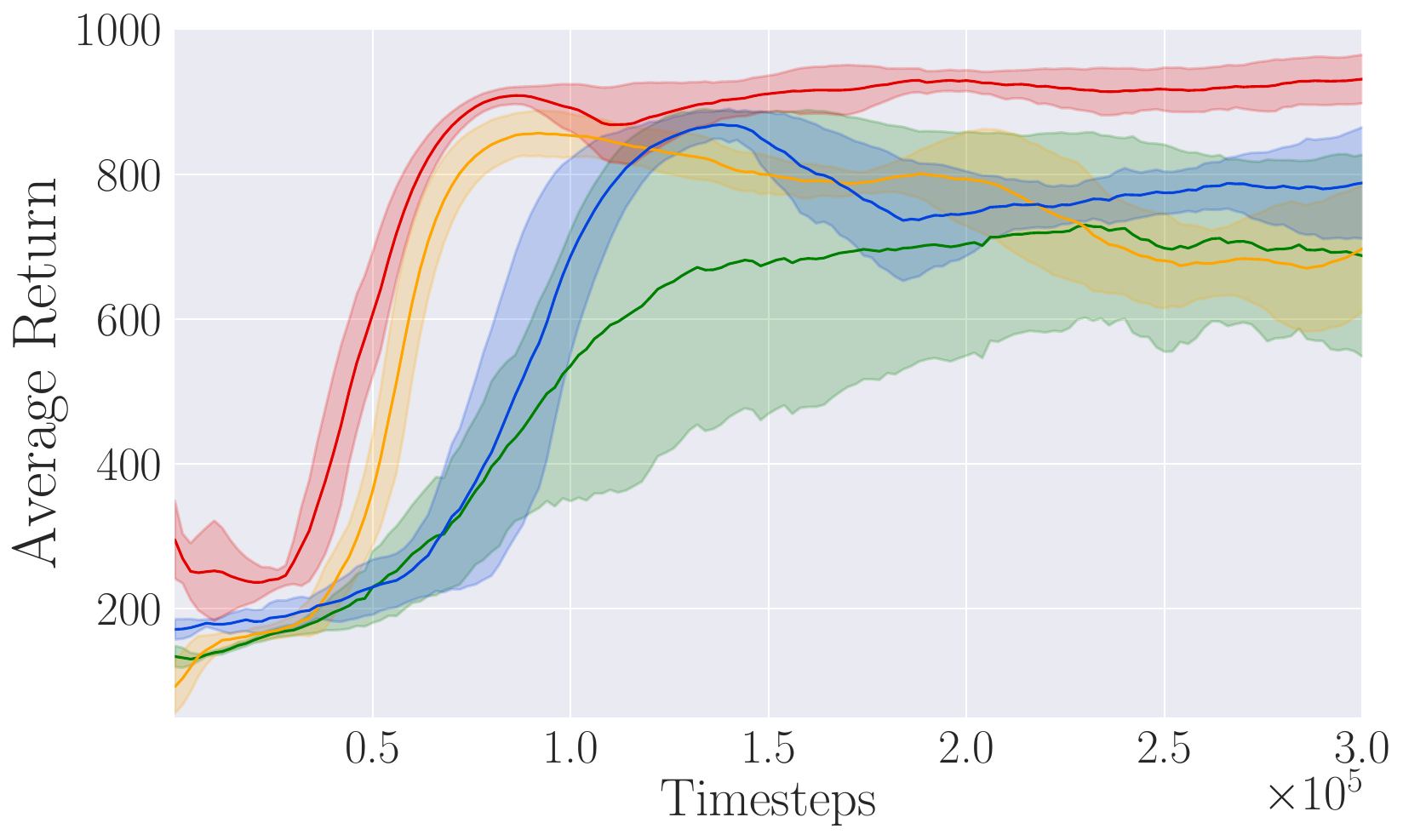} \label{fig:TD3}} \hspace{-0.45cm} &
			\subfigure[SAC]{\includegraphics[width=0.33\textwidth]{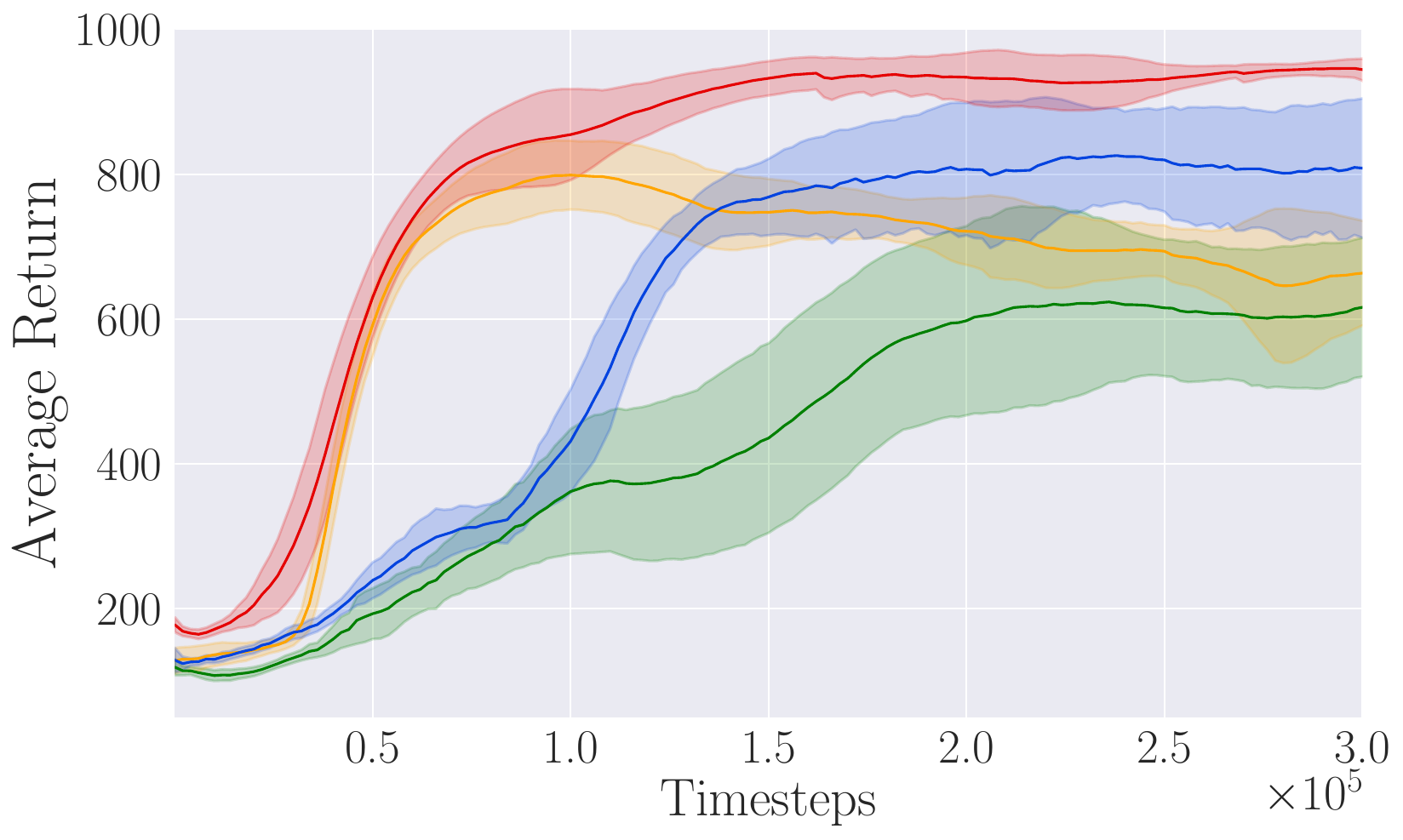} \label{fig:SAC}} \hspace{-0.45cm} &
		\end{tabular}
		\vspace*{-0.2cm}
		\caption{Benchmarks of four RL frameworks trained with (a) PPO, (b) TD3, and (c) SAC RL algorithms.
			Each plot depicts the learning curves for Mono-MLP (green), Mono-\textbf{EMLP} (blue), Mod-MLP (orange), and Mod-\textbf{EMLP} (red).}
		\label{fig:learning_curves}
		\vspace*{-0.2cm}
	\end{figure*}
	
	\paragraph*{Hyperparameters}
    The weighting factors in the reward are chosen as follows.
	To prioritize accurate trajectory tracking, the coefficients for position and yaw errors are assigned higher values, specifically $k_v = k_{b_1} = 6.0$.
    And other factors are chosen as	$k_v = 0.4$, $k_\Omega = k_{\omega_{12}} = 0.6$, and $k_{\Omega_3} = 0.1$, respectively.
	For the integral terms, both $k_{I_x}$ and $k_{I_{b_1}}$ were set to 0.1, with $\alpha = 0.1$ and $\beta = 0.05$, respectively.
	For action policy regularization, we set $\lambda_T = 0.4$, $\lambda_S = 0.3$, and $\lambda_M = 0.6$, respectively.
	To further enhance training robustness and prevent overfitting, we adopted the SGDR learning rate scheduler~\cite{loshchilov2016sgdr} and linearly decayed the exploration noise.
    The other hyperparameters for training are summarized at \Cref{tab:hyper_params}.
		
	\subsection{Training Curves} 

	\Cref{fig:learning_curves} presents the learning curves of the four architectures, namely Mono-MLP (green), Mod-MLP (orange), Mono-EMLP (blue), and Mod-EMLP (red), for each of three RL techniques of PPO, TD3, and SAC, where the average return is illustrated over training timesteps for 10 random seeds.
	During evaluation, the evaluation reward is computed by $r_\textrm{eval} = -\|e_x\| - |e_{b_1}|$ at each step and then normalized to the range $[0, 1]$ to ensure fair comparisons across all frameworks.
	The resulting average returns are reported at every 2,000 timesteps over 10 test trajectories without exploration noise.
	The curves are smoothed using an exponential moving average with a smoothing weight of 0.85, where solid lines and shaded areas represent the mean and the $2\sigma$ bounds across random seeds, respectively.
	
    \paragraph*{Equivariance vs. Non-equivariance}
    These results demonstrate that both EMLP-based models, Mod-EMLP (red) and Mono-EMLP (blue), converge faster and achieve higher returns than their respective MLP-based counterparts, Mod-MLP (orange) and Mono-MLP (green).
    This underscores the advantage of leveraging equivariant learning mechanisms, which enable RL models to capture the inherent symmetries of the quadrotor control problem, thereby enhancing sample efficiency and generalization capability.
    Notably, Mod-EMLP outperforms all other frameworks early in training, achieving the highest returns with fewer samples.
    In contrast, the MLP-based architectures, Mono-MLP and Mod-MLP, exhibit slower learning rates and require more timesteps to achieve comparable rewards, highlighting the limitations of non-equivariant configurations for complex control tasks.

    \paragraph*{Monolithic vs. Modular}
    By incorporating equivariant learning within a modular design, Mod-EMLP achieves the most efficient policy learning, resulting in superior early-stage performance.
    This demonstrates the advantage of the modular frameworks (red and orange), where two agents learn translational and yawing motions in parallel, accelerating convergence and enhancing flight performance.
    Conversely, monolithic architectures (blue and green) show slower progress compared to their modular counterparts, leading to suboptimal performance and overfitting during extended training periods.
    This slower convergence is likely due to the inherent challenge of training monolithic policies, which require simultaneous learning of all control aspects.

	\subsection{Flight Experiments}	
	
	\begin{figure}[t]
		\centering
		\includegraphics[width=0.47\textwidth]{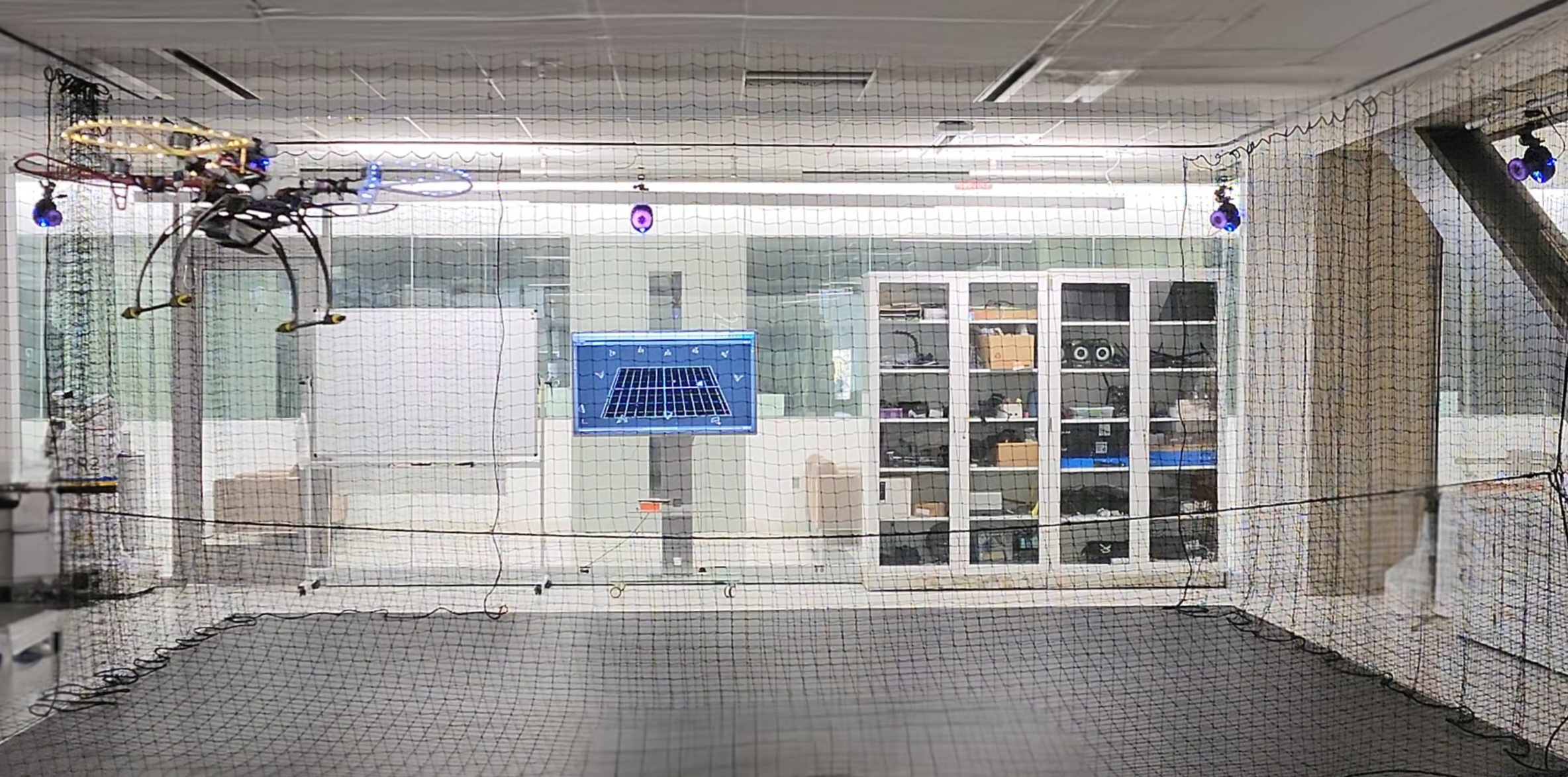}
		\caption{
			Experimental setup for indoor flight testing at the Flight Dynamics and Control Lab, GWU.
			A custom-designed quadrotor UAV, developed for sim-to-real transfer, is tracked by a 12-camera Vicon Valkyrie VK-8 motion capture system. 
			Position and attitude measurements are transmitted to the on-board NVIDIA Jetson TX2 at 200 Hz via Wi-Fi.}
		\label{fig:gwu_fdcl}
		\vspace*{-0.2cm}
	\end{figure}

	To evaluate the flight performance of the proposed equivariant reinforcement learning frameworks, we compared their performance with their non-equivariant baselines through a series of flight experiments in both numerical simulations and the real world.
	
	Frist, we developed the simulation environment in Python based on the quadrotor equations of motion \eqref{eqn:x_dot}--\eqref{eqn:W_dot}, which are discretized using a numerical integration scheme using the OpenAI Gym library \cite{brockman2016openai}, a popular toolkit for RL training environment development. 
	The PPO, TD3, and SAC algorithms were implemented using PyTorch \cite{paszke2019pytorch} for efficient training and deployment of RL models.
	The simulator operates at a frequency of 200 Hz, matching the real-world setup to facilitate the seamless transfer of policies from simulation to hardware.

	Next, a hardware platform was developed in the Flight Dynamics and Control Lab to deploy and test trained RL control policies in real-world environments, integrating a flight computer, sensors, and actuators.
	Connections between these components are facilitated by a custom-designed PCB, ensuring proper voltage regulation and signal integrity. 
	An \textit{NVIDIA Jetson TX2} serves as the onboard flight computer, managing sensor data processing, state estimation, and control computations. 
	For motion measurement, a Vicon motion capture system with twelve \textit{Valkyrie VK-8} cameras tracks reflective markers on the quadrotor frame for precise position and attitude estimation during indoor tests, transmitting data to the onboard computer via Wi-Fi at 200 Hz. 
	Additionally, a \textit{VectorNav VN100} inertial measurement unit (IMU) provides acceleration and angular velocity data at 200 Hz.
	The actuator system comprises four \textit{T-Motor 700KV} brushless DC motors paired with \textit{MS1101} polymer propellers, powered by a 14.8V 4-cell LiPo battery. 
	Motor commands are sent via I2C communication to \textit{MikroKopter BL-Ctrl v2} electronic speed controllers (ESCs), which regulate motor speeds effectively.
	
	In addition, the custom flight software was designed for robust and reliable operation, consisting of sensor processing, state estimation, and control modules. 
	A multi-threaded architecture in C++ is used for sensor and estimator operations, ensuring real-time performance. 
	Sensor measurements are processed asynchronously by dedicated threads and buffered in a thread-safe first-in, first-out (FIFO) queue.
	An Extended Kalman Filter (EKF), running in a separate thread, fuses sensor raw data to estimate the state vector $(\hat{x}, \hat{v}, \hat{R}, \hat{\Omega})$, which is then shared with the RL control module and the ground station.
	The RL controller, implemented in Python, deploys pre-trained PyTorch models and interfaces with C++ modules through ROS2.
	This module computes motor commands based on the estimated states and user-provided commands.
	A ground station serves as a communication hub, sending trajectory commands and receiving flight data for real-time monitoring. 
	This integrated architecture minimizes latency and maximizes robustness, facilitating effective sim-to-real policy transfer.
	
	\paragraph*{Flight Performance Comparison}
	First, we assess the trajectory tracking performance of each framework at a desired yaw rate of 20 $\mathrm{deg/s}$ in both simulation and real-world environments.
	\Cref{tab:comparison} provides a quantitative summary of the average root-mean-square errors (RMSEs) in position $\bar{e}_x$, velocity $\bar{e}_v$, angular velocity $\bar{e}_\Omega$, and heading $\bar{e}_{b_1}$.
	Additionally, we report the average total force $\bar{f}$ and maximum yawing moment $\max |M_3|$, which are critical metrics for stability and energy efficiency.

	\Cref{fig:Lissajous} further visualizes a figure-eight Lissajous trajectory of the real-world flight, with the reference trajectory shown as black-dotted curves.
	The subplots compare the tracking accuracy of all frameworks, with blue hues corresponding to smaller yawing errors and green hues to larger errors. 

	\begin{table*}[t]
		\centering
		\setlength{\tabcolsep}{1.1em} 
		\renewcommand{\arraystretch}{1.35}  
		\caption{Comparison of flight performance for each RL framework for the desired yaw rate of 20 $\mathrm{deg/s}$, quantified by the root mean squared error (RMSE), in position, $\bar{e}_x$, velocity, $\bar{e}_v$, angular rate, $\bar{e}_\Omega$, and heading, $\bar{e}_{b_1}$, alongside the average total force, $\bar{f}$, and moments, $\bar{M}$. 
			Best numbers are indicated in bold.}
		\label{tab:comparison}
		\begin{tabular}{c|c|c|c|c|c|c|c}
			\toprule
			Framework & Environment & $\bar{e}_x (\mathrm{cm})$ & $\bar{e}_v (\mathrm{cm/s})$ & $\bar{e}_\Omega (\mathrm{deg/s})$ & $\bar{e}_{b_1} (\mathrm{deg})$ & $\bar{f} (\mathrm{N})$ & $\max |M_3| (\mathrm{Nm})$ \\ \midrule
			& Simulation & 14.50 & 20.04 & 9.69 & 4.47 & 21.49 & 0.126 \\
			\multirow{-2}{*}{Monolithic Policy with MLP} & \cellcolor[HTML]{EFEFEF}Real World & \cellcolor[HTML]{EFEFEF}14.63 & \cellcolor[HTML]{EFEFEF}18.15 & \cellcolor[HTML]{EFEFEF}10.28 & \cellcolor[HTML]{EFEFEF}7.15 & \cellcolor[HTML]{EFEFEF}22.75 & \cellcolor[HTML]{EFEFEF}0.134 \\ \midrule
			& Simulation & 11.08 & 14.19 & 7.48 & 0.35 & 21.46 & 0.012 \\ \cline{3-4}
			\multirow{-2}{*}{Modular Policies with MLP} & \cellcolor[HTML]{EFEFEF}Real World & \cellcolor[HTML]{EFEFEF}10.95 & \cellcolor[HTML]{EFEFEF}14.48 & \cellcolor[HTML]{EFEFEF}\textbf{9.59} & \cellcolor[HTML]{EFEFEF}4.46 & \cellcolor[HTML]{EFEFEF}22.83 & \cellcolor[HTML]{EFEFEF}0.066 \\ \midrule
			& Simulation & 9.44 & 12.50 & 7.67 & 1.30 & 21.43 & 0.016 \\
			\multirow{-2}{*}{Monolithic Policy with EMLP} & \cellcolor[HTML]{EFEFEF}Real World &	\cellcolor[HTML]{EFEFEF}9.70 & \cellcolor[HTML]{EFEFEF}14.91 & \cellcolor[HTML]{EFEFEF} 11.55 & \cellcolor[HTML]{EFEFEF}5.70 & \cellcolor[HTML]{EFEFEF}22.68 & \cellcolor[HTML]{EFEFEF}0.088\\ \midrule
			& Simulation & \textbf{7.54} & \textbf{10.16}& \textbf{7.02} & \textbf{0.29} & 21.41 & 0.012 \\
			\multirow{-2}{*}{Modular Policies with EMLP} & \cellcolor[HTML]{EFEFEF}Real World &\cellcolor[HTML]{EFEFEF}\textbf{8.44}& \cellcolor[HTML]{EFEFEF}\textbf{12.52} & \cellcolor[HTML]{EFEFEF}12.93 & \cellcolor[HTML]{EFEFEF}\textbf{4.33} & \cellcolor[HTML]{EFEFEF}22.99 & \cellcolor[HTML]{EFEFEF}0.080\\ \bottomrule
		\end{tabular}
		\vspace*{-0.3cm}
	\end{table*}
	\begin{figure*}[t] 
		\centering
		\begin{tabular}{cccc}
			\hspace{-0.5cm}
			\subfigure{\includegraphics[width=0.24\textwidth]{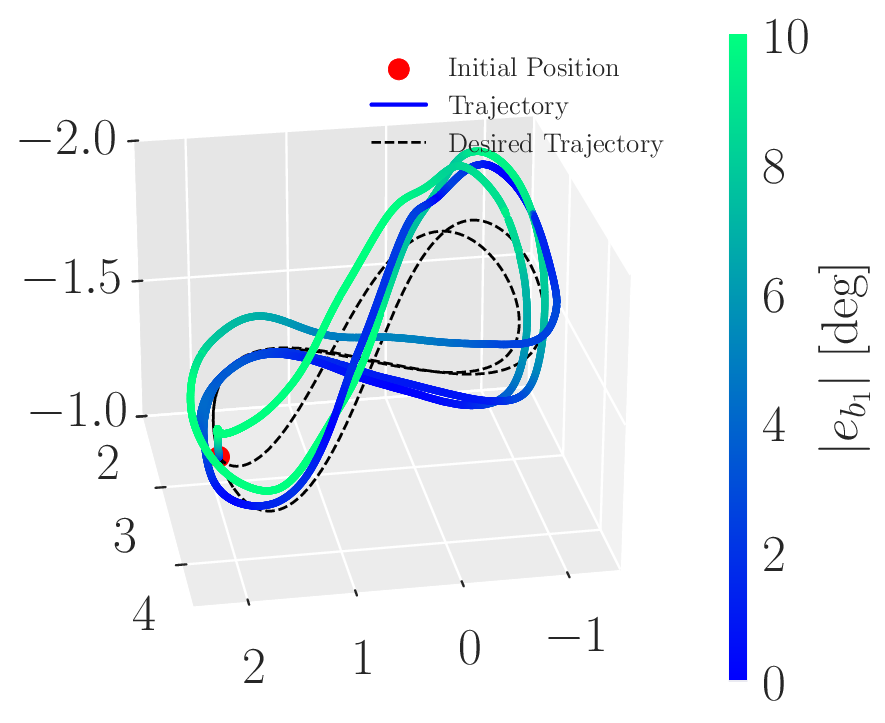} \label{fig:XYZ_mono_mlp}} \hspace{-0.7cm} &
			\subfigure{\includegraphics[width=0.24\textwidth]{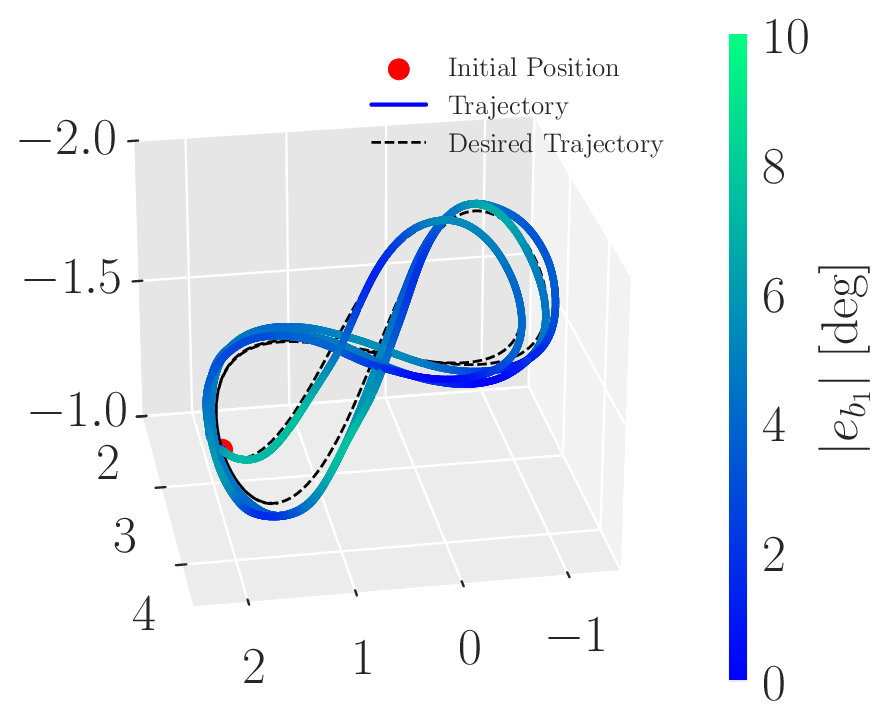} \label{fig:XYZ_mod_mlp}} \hspace{-0.7cm} &
			\subfigure{\includegraphics[width=0.24\textwidth]{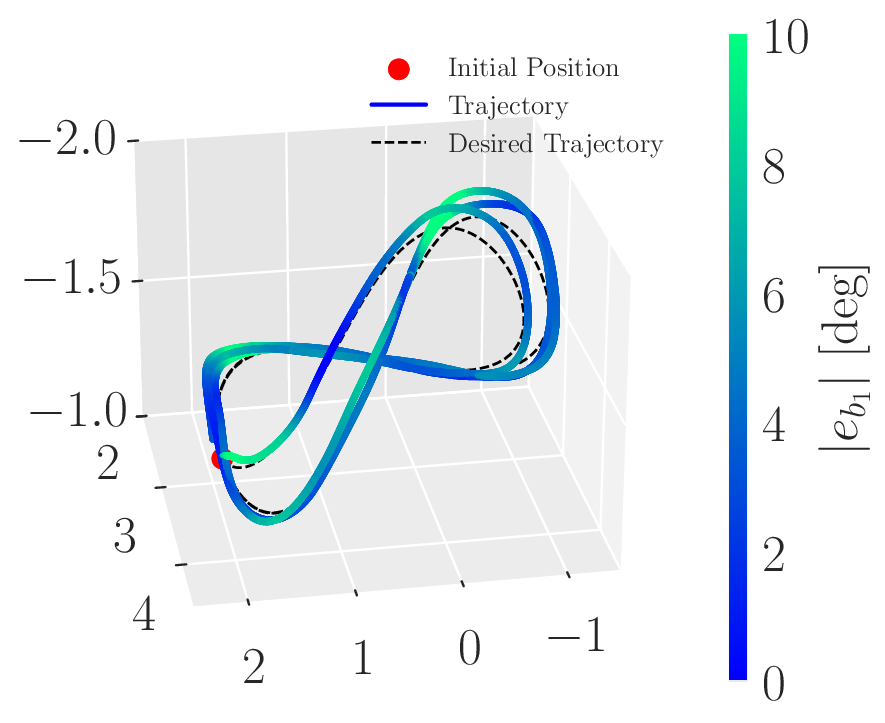} \label{fig:XYZ_mono_emlp}} \hspace{-0.7cm} &
			\subfigure{\includegraphics[width=0.24\textwidth]{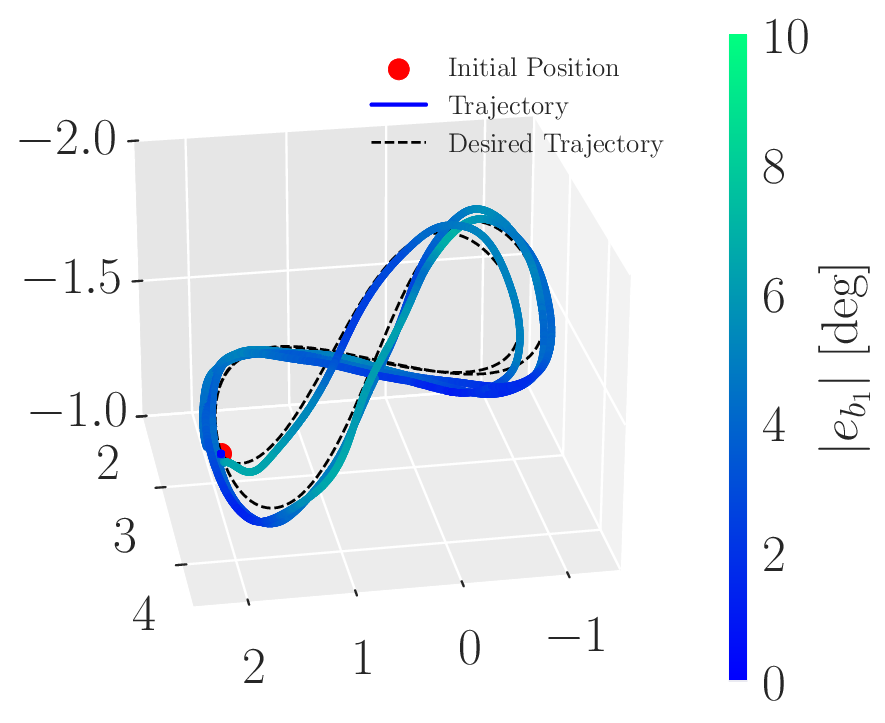} \label{fig:XYZ_mod_emlp}} \hspace{-0.7cm} \\
			\hspace{-1.1cm} \makebox[0.24\textwidth]{\footnotesize{(a) Mono\_MLP}} \hspace{-0.7cm} &
			\hspace{-0.7cm} \makebox[0.24\textwidth]{\footnotesize{(b) Mod\_MLP}} \hspace{-0.7cm} &
			\hspace{-0.5cm} \makebox[0.24\textwidth]{\footnotesize{(c) Mono\_EMLP}} \hspace{-0.7cm} &
			\hspace{-0.5cm} \makebox[0.24\textwidth]{\footnotesize{(d) Mod\_EMLP}} \hspace{-0.7cm} 
		\end{tabular}
		\caption{Real-world flight trajectory of different frameworks corresponding to \Cref{tab:comparison} results. 
			Each plot depicts three laps starting from the same location (red dot), with the reference trajectory shown as black dotted curves.
			The yawing errors are visualized using a colormap, where blue hues correspond to smaller errors and green hues to larger errors.}
		\label{fig:Lissajous}
		\vspace*{-0.2cm}
	\end{figure*}
	\begin{figure*}[t] 
		\centering
		\hspace*{-0.2cm}
		\begin{tabular}{cc}
			\subfigure[$x,v,\Omega$]{\includegraphics[width=0.49\textwidth]{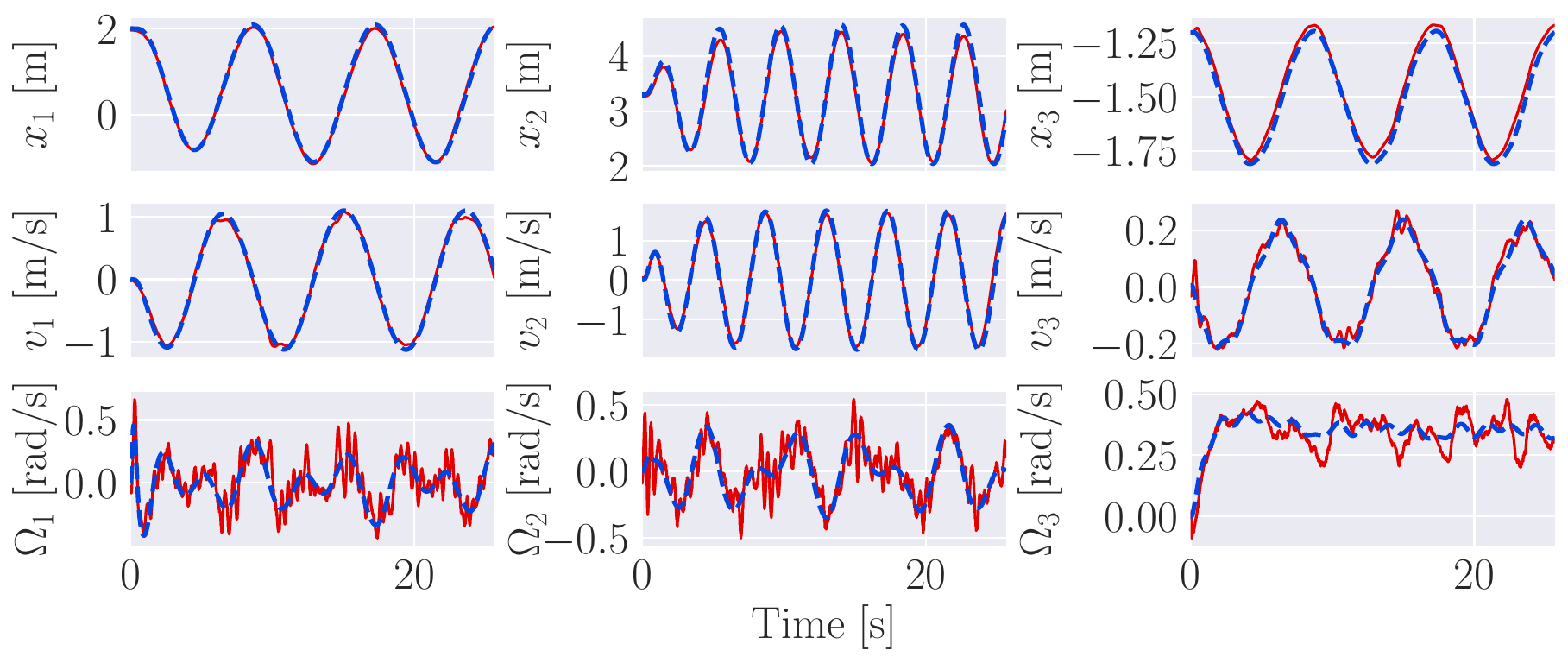} \label{fig:s2r_x_v_W}} \hspace{-0.4cm} &
            \subfigure[{$R=[b_1, b_2, b_3]$}]{\includegraphics[width=0.49\textwidth]{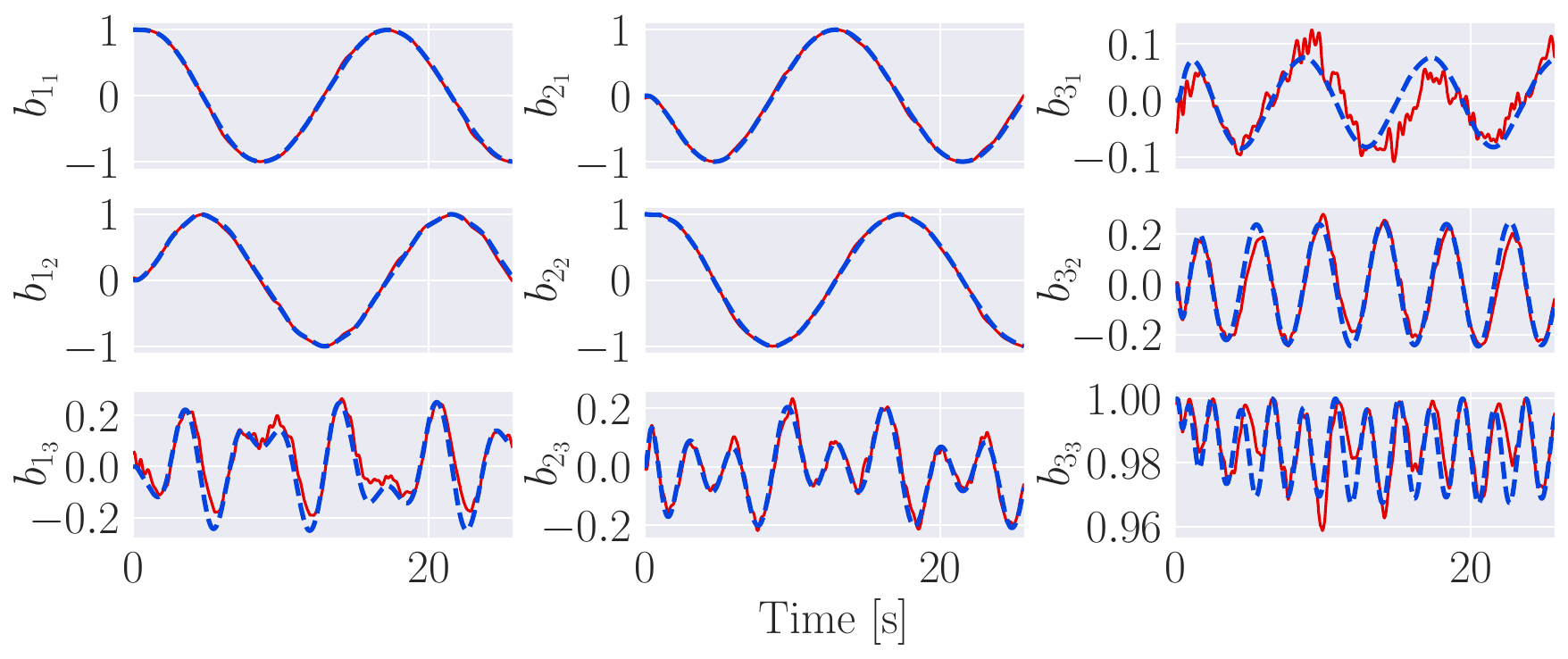} \label{fig:s2r_R}}
		\end{tabular}
		\vspace*{-0.2cm}
		\caption{Validation of zero-shot sim-to-real transfer through comparison of simulated (blue dotted lines) and actual (red lines) flight trajectories. 
			Subfigures (a) show position, velocity, and angular velocity trajectories, while subfigures (b) depict the nine elements of attitude $R$.}
		\label{fig:sim2real}
		\vspace*{-0.2cm}
	\end{figure*}

	It is shown that Mod-EMLP achieves the smallest tracking errors in both position and heading angle, followed by Mono-EMLP, in numerical simulations and flight experiments. 
	This illustrates that the proposed equivariant RL framework not only improves training efficiency but also enhances control performance.
	These findings highlight the advantages of leveraging equivariant learning mechanisms to capture the underlying symmetries of the control problem, resulting in more precise and efficient trajectory tracking.
	
	A further comparison between Mono-EMLP and Mod-EMLP provides additional insights into the benefits of a modular design.
	Although Mono-EMLP demonstrates effective position tracking, it exhibits noticeable yaw errors during real-world flights, as indicated by greener hues in \Cref{fig:XYZ_mono_emlp}.
	This underscores the challenges associated with monolithic architectures in simultaneously controlling roll, pitch, and yaw dynamics, which inherently have distinct characteristics.
	Such limitations often result in suboptimal performance and reduced robustness, as a single agent must handle all control aspects in a coupled manner.
	
	In contrast, Mod-EMLP separates the yaw dynamics from the translational motion dynamics, structuring them as separate modules. This ensures that changes in one module (e.g., perturbations affecting yaw) do not propagate to the other.
	This structural advantage enhances robustness and fault tolerance, as each agent specializes in its designated subtask.
	As a result, Mod-EMLP maintained more precise trajectory tracking with smaller yaw errors, as evident in \Cref{fig:XYZ_mod_emlp}, where the plot color is predominantly blue, emphasizing the benefits of the modular framework.
		
	\paragraph*{Zero-Shot Sim-to-Real Transfer}
	Lastly, to demonstrate the practical utility of our proposed frameworks, we validate their zero-shot sim-to-real transfer capabilities by comparing simulated and real-world flight trajectories.
	This capability is crucial for safely deploying policies trained in simulation to real-world environments, where real-world training poses significant risks to the quadrotor aerial vehicle.
	
	In particular, \Cref{fig:sim2real} illustrates the validation results of Mod-EMLP. Subfigure (a) shows the position, velocity, and angular velocity trajectories, while subfigure (b) presents the nine elements of the rotation matrix $R$.
	These results demonstrate the consistency between simulated and real-world flight trajectories, overcoming discrepancies in real-world conditions such as sensor noise and aerodynamic disturbances.
	This underscores the effectiveness of our sim-to-real strategies, including domain randomization and policy regulation, highlighting their practical applicability.

	\section{Conclusions} \label{sec:Conc}	
	This paper presents data-efficient equivariant reinforcement learning strategies for quadrotor control by leveraging the inherent symmetries of quadrotor dynamics.
	By embedding rotational and reflectional symmetries directly into both monolithic and modular RL frameworks, the proposed methods significantly reduce the dimensionality of the state-action space, leading to faster convergence to optimal policies.
	Experimental results demonstrate that these frameworks outperform non-equivariant baselines in terms of sample efficiency and control performance.
	By validating the advantages of symmetry-aware learning in both simulation and real-world settings, this work provides valuable insights into the application of geometric deep learning, laying the groundwork for broader adoption of equivariant RL frameworks across diverse robotic systems.

	\bibliography{BibTeX}
	\bibliographystyle{IEEEtran}

\end{document}